\documentclass[sigconf]{acmart}
\usepackage[capitalize,noabbrev]{cleveref}
\theoremstyle{plain}
\newtheorem{theorem}{Theorem}[section]

\newtheorem{lemma}[theorem]{Lemma}

\theoremstyle{definition}
\newtheorem{definition}[theorem]{Definition}

\theoremstyle{remark}

\usepackage[textsize=tiny]{todonotes}

\usepackage[utf8]{inputenc} 
\usepackage[T1]{fontenc}    
\usepackage{hyperref}       
\usepackage{url}            
\usepackage{booktabs}       
\usepackage{nicefrac}       
\usepackage{microtype}      
\usepackage{xcolor}         
\usepackage{wrapfig}
\usepackage{mathtools}
\usepackage{caption}
\usepackage{graphicx}  
\usepackage{color}
\usepackage{natbib}
\usepackage{algorithm}
\usepackage{algorithmic}
\usepackage{calc}
\usepackage{multirow}

\usepackage{bm}

\usepackage{balance} 
\newcommand{\eg}{e.\,g.}
\newcommand{\ie}{i.\,e.}
\newcommand{\method}{CARE}

\AtBeginDocument{%
  }

\copyrightyear{2025}
\acmYear{2025}
\setcopyright{cc}
\setcctype{by-nc}
\acmConference[WWW '25]{Proceedings of the ACM Web Conference 2025}{April
28-May 2, 2025}{Sydney, NSW, Australia}
\acmBooktitle{Proceedings of the ACM Web Conference 2025 (WWW '25), April
28-May 2, 2025, Sydney, NSW, Australia}
\acmDOI{10.1145/3696410.3714575}
\acmISBN{979-8-4007-1274-6/25/04}

\begin{document}

\title{Cluster Aware Graph Anomaly Detection}





\author{Lecheng Zheng}
\email{lecheng4@illinois.edu}
\affiliation{%
  \institution{University of Illinois Urbana-Champaign}
  \country{United States}
}

\author{John Birge}
\email{john.birge@chicagobooth.edu}
\affiliation{%
  \institution{University of Chicago}
  \country{United States}
}

\author{Haiyue Wu}
\email{haiyuew2@illinois.edu}
\affiliation{%
  \institution{University of Illinois Urbana-Champaign}
  \country{United States}
}

\author{Yifang Zhang}
\email{zhang303@illinois.edu}
\affiliation{%
 \institution{University of Illinois Urbana-Champaign}
 \country{United States}
}

\author{Jingrui He}
\email{jingrui@illinois.edu}
\affiliation{%
  \institution{University of Illinois Urbana-Champaign}
  \country{United States}
}


\begin{abstract}
Graph anomaly detection has gained significant attention across various domains, particularly in critical applications like fraud detection in e-commerce platforms and insider threat detection in cybersecurity. Usually, these data are composed of multiple types (e.g.,  user information and transaction records for financial data), thus exhibiting view heterogeneity. However, in the era of big data, the heterogeneity of views and the lack of label information pose substantial challenges to traditional approaches. Existing unsupervised graph anomaly detection methods often struggle with high-dimensionality issues, rely on strong assumptions about graph structures or fail to handle complex multi-view graphs. To address these challenges, we propose a cluster aware multi-view graph anomaly detection method, called \method. Our approach captures both local and global node affinities by augmenting the graph's adjacency matrix with the pseudo-label (i.e., soft membership assignments) without any strong assumption about the graph. To mitigate potential biases from the pseudo-label, we introduce a similarity-guided loss. Theoretically, we show that the proposed similarity-guided loss is a variant of contrastive learning loss, and we present how this loss alleviates the bias introduced by pseudo-label with the connection to graph spectral clustering. Experimental results on several datasets demonstrate the effectiveness and efficiency of our proposed framework. Specifically, \method\ outperforms the second-best competitors by more than 39\% on the Amazon dataset with respect to AUPRC and 18.7\% on the YelpChi dataset with respect to AUROC. The code of our method is available at the GitHub link: \url{https://github.com/zhenglecheng/CARE-demo}.

\end{abstract}

\ccsdesc[500]{Security and privacy~Intrusion/anomaly detection and malware mitigation}
\ccsdesc[500]{Computing methodologies~Unsupervised learning}
\ccsdesc[500]{Mathematics of computing~Graph theory}

\keywords{Graph Anomaly Detection; Contrastive Learning; Multi-view Learning}

\maketitle

\section{Introduction}
\label{sec:intro}
Graph-based anomaly detection has been an important research area across diverse domains for decades, particularly within high-impact applications, such as fraud detection within e-commerce platforms~\cite{DBLP:conf/kdd/RamakrishnanSLS19, abdallah2016fraud, wang2019semi} and insider threat detection in the cybersecurity domain~\cite{jiang2019anomaly, brdiczka2012proactive, eberle2010insider}. For instance, in the realm of e-commerce, leveraging a graph-based anomaly detection algorithm proves invaluable for identifying fraudulent sellers by analyzing the properties (i.e., attributes) and connections (i.e., structure) among users~\cite{DBLP:conf/cikm/JinLZCLP21}. Similarly, in the context of insider threat detection, constructing a graph based on users' activities allows investigators to discern anomalous users in the organization by exploring the substructure of the graph~\cite{DBLP:conf/sp/GlasserL13}.

In the era of big data, the collected data often exhibit heterogeneous views (\eg, various data)~\cite{zheng2024online, 
 DBLP:conf/www/ZhengCHC24, DBLP:conf/kdd/QiBH23, DBLP:conf/kdd/QiBH22} and lack labeled data. For example, in the e-commerce platform, multiple types of data can be collected for heterogeneous graph construction, including the user's shopping history, search trends, and product ratings~\cite{jilkova2021digital}; in credit card fraud detection, the data are composed of both cardholder information and transaction records (\eg, online purchase records)~\cite{zhong2020financial}. However, obtaining labels is often impractical due to the expense associated with labeling services and the demand for domain-specific expertise in discerning malicious patterns~\cite{DBLP:conf/kdd/ZhengXZH22}. This highlights the need for innovative approaches that can deal with the intricacies of heterogeneous and unlabeled datasets. 

Until then, many unsupervised anomaly detection methods~\cite{guan2009fast, tong2011non, DBLP:journals/tkde/ChenZZDSZXKT23} have been proposed and they can be categorized into two branches, including feature reconstruction methods and self-supervised learning methods. The feature reconstruction approaches focus on minimizing the reconstruction error of node attributes or structures~\cite{DBLP:conf/sdm/DingLBL19, ahmed2021neighborhood, DBLP:conf/aaai/Cheng0L21}. However, these feature reconstruction based methods tend to suffer from the curse of high dimensionality~\cite{donoho2000high}, especially for citation network where the word occurrence is extracted as the node attributes. Another direction is the self-supervised learning methods~\cite{DBLP:conf/nips/QiaoP23, zheng2021generative, ahmed2021neighborhood, DBLP:conf/pkdd/HuangPMP22, DBLP:journals/tnn/LiuLPGZK22, DBLP:conf/pakdd/XuHZDL22}, which aim to design a proxy task related to anomaly detection, whereas these methods tend to have strong assumptions regarding the graph structure, thus only performing well on some graphs with certain structures. For instance, TAM~\cite{DBLP:conf/nips/QiaoP23} holds the \textit{one-class homophily assumption} that normal nodes tend to have much stronger affinity with each other than with the abnormal nodes and the authors propose to maximize the local node similarity. However, TAM fails to consider a situation where the normal nodes and their normal neighbors might
come from different classes, indicating the distinct features for these
connected normal nodes. A general limitation for both branches is that many of these methods~\cite{zheng2021generative, DBLP:conf/sdm/DingLBL19, DBLP:conf/kdd/HooiSBSSF16} are designed for single-view graphs, suffering from the presence of multiple views. 

To address these limitations, we propose a unified \underline{C}luster \underline{A}wa\underline{RE} graph anomaly detection method to identify anomalous nodes in multi-view graphs, named \method. In this work, we propose to capture both local and global node affinity and design an anomaly score function to assigning higher anomaly scores to nodes that are less similar to their neighbors based on both node attributes and structural similarity. Since the raw adjacency matrix in the graph only contains the local node affinity information, we measure the global node affinity by leveraging the pseudo-label (i.e., soft membership assignments) to augment the original adjacency matrix without any strong assumption about the graph. To reduce the potential bias introduced by the pseudo-label during the optimization, we propose a similarity-guided loss that utilizes the soft assignment to build a similarity map to help the model learn robust representations. We analyze that the proposed similarity-guided loss is a variant of the contrastive loss, and we present how the proposed regularization mitigates the potential bias by connecting it with graph spectral clustering. Our main contributions are summarized below:
\begin{itemize}
    \item A novel self-supervised framework for detecting anomalies in multi-view graphs.
    \item A novel similarity-guided contrastive loss for learning graph contextual information and its theoretical analysis showing the connection to graph spectral clustering.
    \item Theoretical analysis showing the negative impact of other types of contrastive loss.
    \item Experimental results on six datasets demonstrating the effectiveness and efficiency of the proposed framework.
\end{itemize}
The rest of this paper is organized as follows. After briefly reviewing the related work in Section 2, we then introduce a multi-view graph anomaly detection framework in Section 3. Next, we conduct the systematic evaluation of the proposed framework on several datasets in Section 4 before we conclude the paper in Section 5.

\section{Related Work}
In this section, we briefly review the related work on clustering and anomaly detection. 

\textbf{Clustering.}
In the past decades, clustering methods gradually evolved from traditional shallow methods, (\eg, Non-Negative Matrix Factorization (NMF) methods~\cite{shen2010non, DBLP:journals/tkde/AllabLN17, DBLP:conf/aaai/ZhangZLY15} and spectral clustering methods~\cite{DBLP:conf/nips/NgJW01, DBLP:conf/iccv/LiL09, DBLP:conf/nips/BachJ03}), to deep learning-based methods~\cite{DBLP:conf/aaai/LiuTZLSYZ22, DBLP:conf/iconip/GuoLZY17, DBLP:conf/www/Bo0SZL020, DBLP:conf/sdm/ZhengZH23}. 
For example, the authors of ~\cite{shen2010non} propose an NMF-based clustering method that models the intrinsic geometrical structure of the data by assuming that several neighboring points should be close in the low dimensional subspace. 
The work in~\cite{DBLP:journals/tkde/AllabLN17} presents a semi-NMF clustering method by taking advantage of the mutual reinforcement between data reduction and clustering tasks. 
Different from these methods, this paper first follows the idea of the graph pooling method~\cite{DBLP:conf/nips/YingY0RHL18} to get the soft assignment, aiming to capture the global node affinity information and regularize the learned representations by the similarity-guided contrastive loss.

\textbf{Anomaly Detection.}
Anomaly detection has been studied for decades~\cite{grubbs1969procedures, ma2021comprehensive, pourhabibi2020fraud, DBLP:conf/aaai/LiZZH21, DBLP:conf/aaai/DuanW0ZHJ0D23}. The increasing demand in many domains, such as financial fraud detection, anomaly detection in cybersecurity, etc., has attracted many researchers' attention, and a variety of outstanding algorithms have been proposed
~\cite{DBLP:conf/cvpr/PangYSH020, DBLP:conf/accv/AkcayAB18, DBLP:conf/ijcai/ChenTSCZ16, DBLP:journals/tii/ZhouHLMJ21, DBLP:conf/cvpr/LiSYP21, DBLP:conf/cvpr/ParkNH20}. 
To name a few, AnoGAN~\cite{DBLP:conf/accv/AkcayAB18} is a generative adversarial network (GAN) based anomaly detection method, which utilizes the generator and the discriminator to capture the normal patterns while the anomalies are detected based on the residual score and the discrimination score. HCM-A~\cite{DBLP:conf/pkdd/HuangPMP22} designs a self-supervised learning loss by forcing the prediction of the shortest path length between pairs of nodes. ComGA~\cite{DBLP:conf/wsdm/Luo0BYZWX22} proposes a community-aware attributed graph anomaly detection method to detect community structure of the graph. 
TAM~\cite{DBLP:conf/nips/QiaoP23} proposes a scoring measure by assigning large score to the nodes that are less affiliated with their neighbors and introduce truncated affinity maximization to reduce the bias during the optimization. In contrast, this paper proposes to capture both local and global node affinity information, and we propose similarity-guided contrastive learning loss to learn robust representations and to mitigate potential bias.

\begin{table}[htp]
\caption{Definition of Symbols}
\vspace{-2mm}
\centering
\scalebox{0.9}{
\begin{tabular}{|*{2}{c|}}
\hline \textbf{Symbols} & Definition\\
\hline $\bm{V}$         & The set of vertices \\   
\hline $\bm{E}^a$       & The set of edges for the $a$-th view \\   
\hline $\bm{A}^a$       & The $a$-th view adjacency matrix\\   
\hline $\bm{X}^a$       & The $a$-th view node attribute matrix \\   
\hline $v (n)$          & The number of views (nodes)\\    
\hline $\bm{D}^a$       & The degree matrix for the $a$-th view\\   
\hline $\bm{H}^{a}$     & The the $a$-th view node representations \\  
\hline $\bm{\hat{A}}$       & The adjacency matrix augmented by cluster similarity\\   
\hline $\bm{M}^a$         & The soft membership matrix for the $a$-th view\\
\hline $\bm{\bar{H}}$   & The average node representations \\  
\hline $\bm{\bar{M}}$         & The average soft membership matrix\\
\hline $\bm{\tilde{A}}$         & The normalized augmented adjacency matrix\\
\hline
\end{tabular}}
\label{table_symbols}
\vspace{-3mm}
\end{table}

\section{Proposed \method\ Framework}
In this section, we present our proposed framework, \method, for multi-view graph anomaly detection. We begin by defining the notation and then introduce cluster-aware node affinity learning alongside the similarity-guided contrastive learning loss. Next, we outline the overall objective function and the inference process for detecting anomalies. Finally, we analyze the limitations of using weakly supervised contrastive loss in our method.

\subsection{Notation}
Throughout this paper, we use regular letters to denote scalars (\eg, $\alpha$), boldface lowercase letters to denote vectors (\eg, $\bm{x}$), and boldface uppercase letters for matrices (\eg, $\bm{A}$). Given an undirected graph $\mathcal{G} = (\bm{V},\bm{E}^1,...,\bm{E}^v, \bm{X}^1,...,\bm{X}^v)$, our objective is to identify anomalous nodes in the graph, where $v$ represents the number of views, $\bm{V}$ consists of $n$ vertices, $\bm{E}^v$ consists of $m^v$ edges, $\bm{X}^v\in \mathbb{R}^{n\times d_v}$ denotes the feature matrix of the $v$-th view and $d_v$ is the feature dimension. For clarity, we denote $u_i$ as node $i$, $\bm{x}_i^v\in \mathbb{R}^{d_v}$ as the node attributes of $u_i$ for the $v$-th view, $\bm{h}_i\in \mathbb{R}^{d}$ as the embedding of node $u_i$ by any type of GNNs and $d$ is the feature dimensionality of the hidden representation. $\bm{H}^v\in \mathbb{R}^{n \times d^v}$ is the node embedding matrix. We let $\bm{A}^v \in \mathbb{R}^{n\times n}$ denote the adjacency matrix of the $v$-th view where $\bm{A}^v_{ij}=1$ iff node $u_i$ and node $u_j$ are connected, $\bm{D}^v \in \mathbb{R}^{n\times n}$ denotes the diagonal matrix of vertex degrees for the $v$-th view, and $\bm{I} \in \mathbb{R}^{n\times n}$ denotes the identity matrix. The symbols are summarized in Table~\ref{table_symbols}.

\begin{figure}[htp]
\centering
\includegraphics[width=0.85\linewidth]{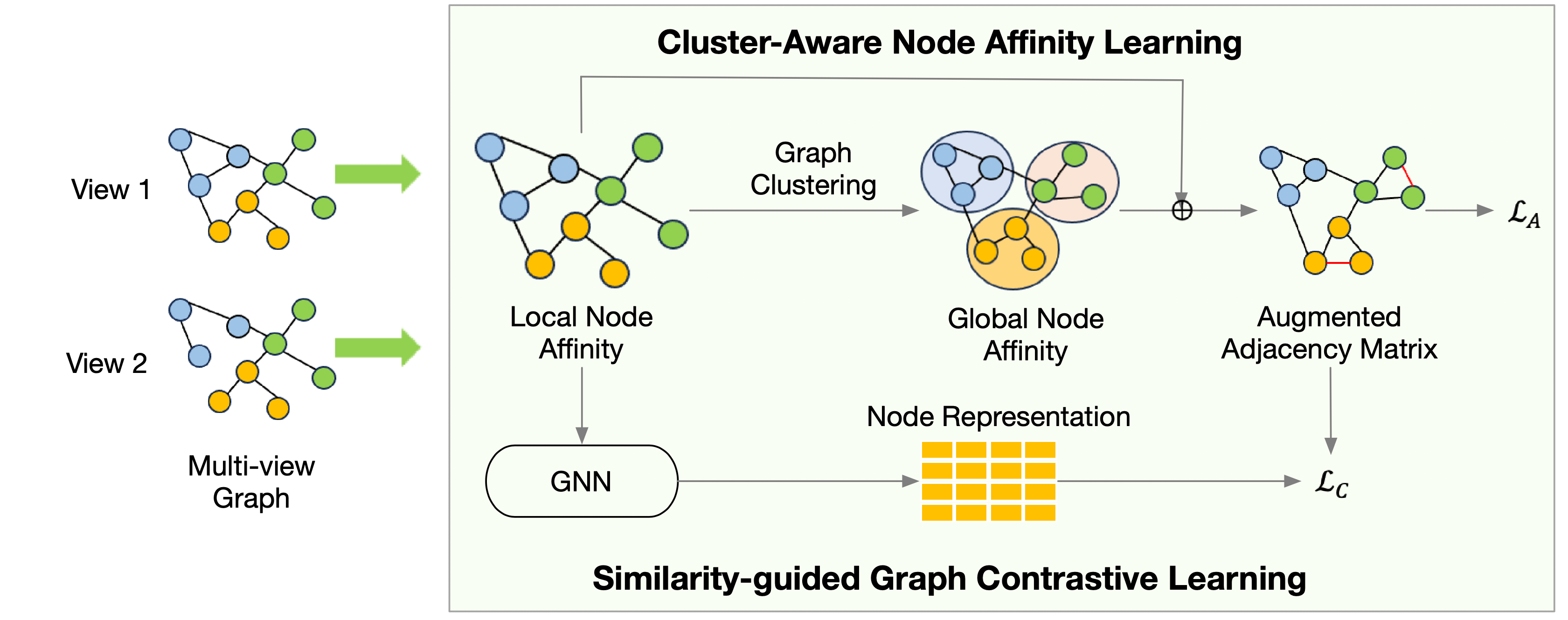}
\caption{The overview of \method. It first extracts the global node affinity based on the soft assignment by graph clustering method, and then combines the global node affinity and local node affinity together. Similarity-guided graph contrastive loss is then introduced to mitigate the potential bias.}
\label{fig:sigil_framework}
\vspace{-3mm}
\end{figure}

\subsection{Cluster-Aware Node Affinity Learning}
\label{graph_pool}
Many self-supervised learning methods~\cite{DBLP:conf/nips/QiaoP23, zheng2021generative, ahmed2021neighborhood, DBLP:conf/pkdd/HuangPMP22, DBLP:journals/tnn/LiuLPGZK22, DBLP:conf/pakdd/XuHZDL22} aim to design a proxy task relevant to the anomaly detection. One branch~\cite{DBLP:conf/nips/QiaoP23, DBLP:journals/corr/abs-2405-17525} is to maximize the local node similarity for a multi-view graph as follows:
\begin{align}
    \label{TAM_loss}
    \nonumber\mathcal{L}_{1} &= \sum_{a=1}^v\sum_{u_i} \frac{1}{|\mathcal{N}^a(i)|} \sum_{{u_j}\in \mathcal{N}^a(i)} \text{sim}(\bm{h}_i^a, \bm{h}_j^a) \\
    & =  \sum_{a=1}^v\sum_{u_i}\sum_{u_j} \frac{1}{\bm{D}_{i}^a} \bm{A}_{ij}^a \cdot \text{sim}(\bm{h}_i^a, \bm{h}_j^a)
\end{align}
where $\text{sim}(\bm{h}_i^a, \bm{h}_j^a)=\frac{\bm{h}_i^a(\bm{h}_j^a)^T}{|\bm{h}_i^a||\bm{h}_j^a|}$ measures the similarity of node embedding between the node $u_i$ and its neighbor $u_j$ based on the $a$-th view of the graph, $\mathcal{N}^a(i)$ denotes the neighbors of the node $u_i$ and $\bm{D}_{i}^a$ is the degree of the node $u_i$ for the $a$-th view. However, these methods typically rely on \textit{one-class homophily assumption}, which posits that normal nodes tend to exhibit strong affinities with each other, while the affinities among abnormal nodes are significantly weaker. This assumption is overly restrictive, as it focuses exclusively on extracting local node affinities while neglecting global node affinities.


A naive solution to relax this constraint is to incorporate high-order information by extending the local $1$-hop neighbors to $k$-hop neighbors. However, this approach presents two main issues. First, including high-order neighbors inevitably introduces more abnormal nodes during the node affinity maximization process, resulting in a sub-optimal solution. Second, it overlooks a crucial scenario where normal anchor nodes and their neighbors may belong to different classes, indicating that these connected normal nodes may have distinct features. Since both normal neighbors from different classes and abnormal neighbors possess features that differ from those of normal anchor nodes, incorporating these distinct features in node affinity learning decreases the likelihood of detecting abnormal nodes. Furthermore, adding high-order neighbors in node affinity learning also increases the probability of including neighbors from different classes, exacerbating the problem. 
To address this issue, we propose incorporating label information into the local node affinity maximization as follows:
\begin{equation}
\label{TAM_label_loss}
    \mathcal{L}_{2} = \sum_{a=1}^v\sum_{u_i}\sum_{u_j} \frac{1}{\sum_{u_j}(\bm{A}_{ij}^a + \bm{S}_{ij}^a)} (\bm{A}_{ij}^a + \bm{S}_{ij}^a) \cdot \text{sim}(\bm{h}_i^a, \bm{h}_j^a)
\end{equation}
where $\bm{S}_{ij}=1$ if node $u_i$ and node $u_j$ belong to the same class and $\bm{S}_{ij}=0$ otherwise. Compared to $\mathcal{L}_{1}$, $\mathcal{L}_{2}$ further encodes the label information in the node affinity learning. However, label information is often unavailable due to the high costs of labeling services and the rapid growth of new data. To address this issue, we propose replacing the unavailable label information with pseudo-labels derived from a graph clustering method. Following the idea of differential graph pooling~\cite{DBLP:conf/nips/YingY0RHL18}, we employ a one-layer Graph Convolutional Network (GCN)~\cite{DBLP:conf/iclr/KipfW17} with a softmax activation function to model soft membership assignments as follows:
\begin{align}
\label{graph_pooling}
\nonumber \bm{M}^a &= \text{GCN}^a(\bm{A}^a, \bm{X}^a, \bm{W}^a) \\
\bm{\bar{M}} &=\frac{1}{v}\sum_{a=1}^v \bm{M}^a
\end{align}
where $\bm{W}^a \in \mathbb{R}^{d_a \times c}$ is the weight matrix of GCN for the $a$-th view and $c$ is the number of clusters and we aggregate soft membership assignments from all views to obtain $\bm{\bar{M}}$. We then augment the adjacency matrix by incorporating the clustering results as follows:
\begin{align}
\bm{\hat{A}} &= (1-\alpha)\bm{\bar{A}} + \alpha \bm{\bar{M}}\bm{\bar{M}}^T
\end{align}
where $\bm{\bar{A}}=\frac{1}{v}\sum_{a=1}^v \bm{A}^a$ and $\alpha\in [0, 1]$ is a hyper-parameter balancing the importance between raw adjacency matrix and the similarity of the soft membership assignments. Our goal is to maximize the cluster-aware node affinity as follows:
\begin{equation}
\label{affinity_loss}
    \mathcal{L}_{A}(u_i) = \sum_{u_j} \frac{1}{\bm{\hat{D}}_i} \bm{\hat{A}}_{ij} \cdot \text{sim}(\bm{\bar{h}}_i, \bm{\bar{h}}_j) - \sum_{a=1}^v r^a||\bm{h}_i^a - \bm{\bar{h}}_i||_2
\end{equation}
where $r^a$ is a learnable weight measuring the importance of the $a$-th view with the condition $\sum_{a=1}^v r^a=1$, $\bm{\hat{D}}_i = \sum_{u_j} \bm{\hat{A}}_{ij}$ measures the degree of node $u_i$ in the new adjacency matrix and $\bm{\bar{h_i}} = \frac{1}{v}\sum_{a=1}^{v}\bm{h}_i^a$ is the average node representation. The second term enforces consistency in node embeddings across all views. However, optimization using Eq.~\ref{affinity_loss} may be significantly biased by low-quality soft membership assignments at the early stages. To address this issue, we introduce the similarity-guided graph contrastive regularization.

\subsection{Similarity-guided Graph Contrastive Regularization}
To mitigate potential bias introduced by low-quality soft membership assignments, we propose a similarity-guided graph contrastive loss that minimizes the difference between the similarity of the soft assignment and the similarity of node representations for any pair of nodes. This is formulated as follows:
\begin{equation}
\label{L_2}
\begin{split}
        \mathcal{L}_2 &=\min_{\bm{\bar{H}}} ||\bm{\bar{M}}\bm{\bar{M}}^T-\bm{\bar{H}}\bm{\bar{H}}^T||_F^2 \\
\end{split}
\end{equation}
where $\bm{\bar{M}}$ is the soft assignment matrix computed in Eq.~\ref{graph_pooling}. 
$\mathcal{L}_2$ aims to learn the hidden representations such that the representations of node $u_i$ and node $u_j$ are expected to be close in the latent space if their soft membership assignments are similar. Following the idea of many existing methods restoring graph structure along with the node attributes of the graph~\cite{DBLP:conf/nips/QiaoP23, zheng2021generative, ahmed2021neighborhood, DBLP:conf/pkdd/HuangPMP22}, we propose to take the graph topological structure into consideration, reformulating the similarity-guided graph contrastive loss as:
\begin{equation}
\label{L_c}
\begin{split}
        \mathcal{L}_C &=\min_{\bar{H}} ||\bm{\tilde{A}}-\bm{\bar{H}}\bm{\bar{H}}^T||_F^2 \\
\end{split}
\end{equation}
where $\bm{\tilde{A}}=\bm{\tilde{D}}^{-1/2}\bm{\hat{A}}\bm{\tilde{D}}^{-1/2}$ is the normalized augmented adjacency matrix, and $\bm{\tilde{D}}\in \mathbb{R}^{n\times n}$ is the diagonal matrix with $\bm{\tilde{D}}_{ii} =\sum_{u_j} \bm{\hat{A}}_{ij}$. \\

\noindent\textbf{Theoretical Justification.}
We aim to provide a theoretical analysis of how the proposed similarity-guided graph contrastive regularization mitigates potential bias by connecting it with graph spectral clustering. Before delving into the direct analysis of bias mitigation, we first demonstrate that $\mathcal{L}_C$ functions as a contrastive learning loss.
  
\begin{lemma}
\label{lemma_1}
(Similarity-guided Graph Contrastive Loss) Let $\bm{\bar{M}}$ be the output of a one-layer graph neural network defined in Eq.~\ref{graph_pooling}. Then, we have 
\begin{equation}
    \mathcal{L}_C = \mathcal{L}_f + C
\end{equation}
where $\mathcal{L}_f = -\sum_{i=1}^n\sum_{j=1}^n \log\frac{\exp(2\bm{\tilde{A}}_{ij}\bm{\bar{h}}_i\bm{\bar{h}}_j^T)}{\Pi_{k=1}^n \exp((\bm{\bar{h}}_i\bm{\bar{h}}_k^T)^2)^{1/n}}$ is a graph contrastive loss and $C$ is a constant.
\end{lemma}
\noindent
\textbf{See proof in Appendix~\ref{appendix_2}.}\\
\textbf{Remark:} Compared to traditional contrastive learning losses, the denominator of $\mathcal{L}_f$ in Lemma~\ref{lemma_1} is the product of the exponential similarity between two node embeddings rather than a summation. The weakly supervised contrastive methods~\cite{DBLP:conf/iccv/Zheng0Y0Z0021, DBLP:conf/iclr/TsaiLLLSM22} impose a constraint that forces a given pair of nodes to form a positive/negative pair based on their likelihood of being assigned to the same cluster (further discussion can be found in Subsection~\ref{Contrastive_regularization}). Unlike these "discrete" formulations regarding positive and negative pairs, our approach is in a continuous form. In $\mathcal{L}_f$, $\bm{S}_{ij}=\bm{\bar{M}}_i\bm{\bar{M}}_j^T$ can be interpreted as the similarity measurement of a node pair $(u_i, u_j)$ in terms of soft assignment, guiding the similarity of the representations of two nodes in the latent space. If we ignore the influence of the adjacency matrix in 
$\bm{\tilde{A}}=\bm{\tilde{D}}^{-1/2}\bm{\hat{A}}\bm{\tilde{D}}^{-1/2}=\bm{\tilde{D}}^{-1/2}(\alpha \bm{\bar{M}}\bm{\bar{M}}^T + (1-\alpha)\frac{1}{v}\sum_{a=1}^v\bm{A})\bm{\tilde{D}}^{-1/2}$
by setting $\alpha$ to be 1 (\ie, $\bm{\tilde{A}}=\bm{\tilde{D}}^{-1/2}\bm{\bar{M}}\bm{\bar{M}}^T\bm{\tilde{D}}^{-1/2}$), it would be interesting to see that when two nodes are sampled from two distant clusters, $\bm{\tilde{A}}_{ij}\approx 0$ and $\log\frac{\exp(2\bm{S}_{ij}\bm{\bar{h}}_i\bm{\bar{h}}_j^T)}{\Pi_{k=1}^n\exp((\bm{\bar{h}}_i\bm{\bar{h}}_k^T)^2)^{\frac{1}{n}}}=0$ for this pair of nodes.
Notice that for the node pair ($u_i$, $u_j$) with high confidence being assigned to the same cluster, the weight $\bm{S}_{ij}$ is larger than the weight $\bm{S}_{ik}$ for the uncertain node pair ($u_i$, $u_k$), where the node $u_k$ has low confidence to be assigned to the same cluster as $u_i$. By reducing the weight for these unreliable positive pairs, the negative impact of uncertain pseudo-labels can be alleviated. 

Next, we aim to clarify our proposed contrastive loss by demonstrating the connection between $\mathcal{L}_C$  and graph spectral clustering.
\begin{lemma}
\label{lemma_2} (Graph Contrastive Spectral Clustering) Let $\bm{\bar{M}}$ be the output of a one-layer graph neural network defined in Eq.~\ref{graph_pooling} and $\bm{\bar{h}}_i$ and $\bm{\bar{h}}_j$ be unit vectors. Then, minimizing $\mathcal{L}_C$ is equivalent to minimizing the following loss function:
\begin{equation}
    \min \mathcal{L}_C = \min [2Tr(\bm{\bar{H}}^T \bm{L} \bm{\bar{H}}) + R(\bm{\bar{H}})] \label{lemma2_1}
\end{equation}
where $\bm{L}=\bm{I}-\bm{\tilde{A}}$ can be considered as the normalized graph Laplacian, $\bm{I}$ is the identity matrix and $R(\bm{\bar{H}})=||\bm{\bar{H}}\bm{\bar{H}}^T||_F^2$ is the regularization term.
\end{lemma}
\noindent\textbf{See proof in Appendix~\ref{appendix_3}.}

\noindent\textbf{Remark:} Based on Lemma~\ref{lemma_2}, $\mathcal{L}_C$ can be considered as the graph spectral clustering. Graph spectral clustering~\cite{DBLP:journals/sac/Luxburg07} aims to find clusters that minimize connections between different clusters while maximizing the connections within each cluster. Traditional graph spectral clustering~\cite{DBLP:journals/neco/BelkinN03} aims to find the embedding $\bm{\bar{H}}$ such that $Tr(\bm{\bar{H}}^T \bm{L}'\bm{\bar{H}})$ is minimized, where $\bm{L}'$ is the normalized graph Laplacian. The first term of Eq.~\ref{lemma2_1} is similar to the objective function in traditional graph spectral clustering, but we enhance it by incorporating the similarity measurement  $\bm{S}_{ij}=\bm{\bar{M}}_i\bm{\bar{M}}_j^T$ into the normalized graph Laplacian. By including the similarity measurement in the graph spectral clustering, we reinforce $\mathcal{L}_C$ to mitigate the bias introduced by the clustering method defined in Eq.~\ref{graph_pooling}. It is important to note that in Lemma~\ref{lemma_2}, the constraint that $\bm{\bar{h}}_i$ and $\bm{\bar{h}}_j$ are unit vectors is a common practice in many existing works~\cite{DBLP:conf/nips/HaoChenWGM21, DBLP:conf/nips/ZhuSK21, DBLP:conf/icml/0001I20}. This constraint can be easily implemented through normalization, \ie, $\bm{\bar{h}}_i=\frac{\bm{\bar{h}}_i}{||\bm{\bar{h}}_i||_2}$. 

\subsection{Objective Function and Inference}
Now, we are ready to introduce the overall objective function:
\begin{equation}
    \label{overall}
    \begin{split}
        \min J &= -\sum_{u_i}\mathcal{L}_A(u_i)  + \lambda \mathcal{L}_C
    \end{split}
\end{equation}
where $\mathcal{L}_A$ is the cluster-aware node affinity loss and $\mathcal{L}_C$ is the similarity-guided graph contrastive loss. $\lambda$ is a constant parameter balancing these two terms. During the inference stage, we directly use the cluster-aware node affinity loss $\mathcal{L}_A$ as the abnormal score:
\begin{equation}
    score_i = -\mathcal{L}_A(u_i) 
\end{equation}

\subsection{Why can't weakly supervised contrastive learning loss be used as a regularization?}
\label{Contrastive_regularization}
One effective remedy to reduce uncertainty and learn high-quality representations in an unsupervised setting is through contrastive learning loss~\cite{zheng2024drgnn, DBLP:conf/emnlp/LanZMK24, DBLP:conf/nips/HaoChenWGM21, DBLP:conf/icml/SaunshiPAKK19, li2024can, jing2021hdmi, jing2022coin, jing2022x, jing2024sterling, jing2024automated}, which has demonstrated significant performance improvements in representation quality. However, an existing study~\cite{DBLP:conf/kdd/ZhengXZH22} has theoretically proved that simply applying vanilla contrastive learning loss (\ie, InfoNCE~\cite{oord2018representation}) can easily lead to the suboptimal solution. Similarly, according to ~\cite{DBLP:conf/icml/0001I20}, both normal and abnormal nodes are uniformly distributed in the unit hypersphere in the latent space by minimizing the vanilla contrastive learning loss, which leads to worse performance for the anomaly detection task. To address this issue, many weakly supervised contrastive losses~\cite{DBLP:conf/iccv/Zheng0Y0Z0021, DBLP:conf/iclr/TsaiLLLSM22, DBLP:conf/kdd/ZhengJLTH24} are proposed by incorporating the semantic information, such as the clustering results, into a contrastive regularization term as follows: 
\begin{equation}
\label{L_1}
    \mathcal{L}_3 = -\sum_{i=1}^n\sum_{j\in C(i), j\neq i}\log\frac{\text{sim}(\bm{\bar{h}}_i, \bm{\bar{h}}_j)}{\text{sim}(\bm{\bar{h}}_i, \bm{\bar{h}}_j) + \sum_{k\notin C(i)}\text{sim}(\bm{\bar{h}}_i, \bm{\bar{h}}_k)}
\end{equation}
where $\bm{\bar{h}}_i=\sum_{a=1}^v \bm{h}_i^a$ is the representation for node $u_i$ aggregated over all views, $\text{sim}(\bm{\bar{h}}_i, \bm{\bar{h}}_j)=\exp(\bm{\bar{h}}_i\bm{\bar{h}}_j^T/\tau)$ and $\tau$ is the temperature. $j\in C(i)$ means that node $u_j$ and node $u_i$ are assigned into the same cluster or form a positive pair, while $k\notin C(i)$ means that node $u_k$ and node $u_i$ are assigned into two different clusters, resulting in a negative pair. The intuition of the above equation is that if two nodes are from the same cluster, they should be close in the latent space by maximizing their similarity. 
However, we find out that the construction of the positive and negative pairs in Eq.~\ref{L_1} heavily relies on the quality of the soft assignment, while directly converting the soft assignment to the binary membership inevitably introduces bias/noise during the training phase. This bias will be amplified further in the node affinity learning as we mentioned in Subsection~\ref{graph_pool}. (We also validate this in the ablation study in Subsection \ref{Ablation_study}.) Here, we theoretically analyze that including this bias in the weakly supervised contrastive loss defined in Eq.~\ref{L_1} leads to suboptimal solution. Formally, we first define what is a true positive pair and a false positive pair respectively, and then introduce Theorem~\ref{theorem_1} to show the issue in Eq.~\ref{L_1}. 

\begin{definition}
    \label{definition_1}
    Given a sample $\bm{x_i}$, we say ($\bm{x}_i$, $\bm{x}_j$) is a true positive pair (or a false negative pair), if their optimal representations satisfy $\exp(\bm{\bar{h}}_i\bm{\bar{h}}_j^T/\tau) > 1$ for a small positive value $\tau$. Similarly, we say ($\bm{x}_i$, $\bm{x}_k$) is a false positive pair (or a true negative pair), if their optimal representations satisfy $\exp(\bm{\bar{h}}_i\bm{\bar{h}}_k^T/\tau) \approx 0$ for a small positive value $\tau$. 
\end{definition}

\begin{theorem}
\textit{
Given the contrastive learning loss function $\mathcal{L}_3$, if there exists one false positive sample in the batch during training, the contrastive learning loss will lead to a sub-optimal solution.}
\label{theorem_1}
\end{theorem}
\noindent\textbf{See proof in Appendix~\ref{appendix_1}.}
\section{Experimental Results}
In this section, we demonstrate the performance of our proposed framework in terms of both effectiveness and efficiency by comparing it with state-of-the-art methods.

\begin{table}[tp]
\centering  
\caption{Statistics of the datasets, including the number of nodes, anomalies, and edges for two views.}
\vspace{-3mm}
\label{TB:Net}
\begin{tabular}{ccccccc}
\hline Name & $|\bm{V}|$  & $|\bm{E}^1|$ & $|\bm{E}^2|$ & \# Anomalies\\
\hline BlogCatalog    & 5,196       & 171,743       & -             & 298\\
\hline Amazon         & 10,244      & 175,608       & -             & 445\\
\hline YelpChi        & 24,741      & 49,315        & -             & 597 \\
\hline CERT           & 1,000       & 24,213        & 22,467        & 70\\
\hline IMDB           & 4,780       & 1,811,262     & 419,883       & 334 \\
\hline DBLP           & 4,057       & 299,499       & 520,440     & 283 \\
\hline
\end{tabular}
\vspace{-3mm}
\end{table}

\subsection{Experiment Setup}
\subsubsection{Datasets:} We evaluate the performance of our proposed framework on six datasets for both single-view and multi-view graph anomaly detection scenarios, including the Insider Threat Test (CERT)~\cite{DBLP:conf/sp/GlasserL13}, DBLP~\cite{DBLP:conf/www/WangJSWYCY19}, IMDB~\cite{DBLP:conf/www/WangJSWYCY19}, BlogCatalog~\cite{DBLP:conf/kdd/TangL09}, Amazon~\cite{DBLP:conf/cikm/DouL0DPY20} and YelpChi~\cite{DBLP:conf/kdd/KumarZL19} datasets. 
Among these datasets, CERT, IMDB, and DBLP are multi-view graphs, while BlogCatalog, Amazon, and YelpChi are single-view graphs. CERT, Amazon, and YelpChi are real-world datasets, whereas IMDB, DBLP, and BlogCatalog are semi-synthetic graphs. (See Appendix \ref{generate_anomaly} for the details of generating anomalous nodes.)
Specifically, the CERT dataset is a collection of synthetic insider threat test datasets that provides both synthetic background data and data from synthetic malicious actors. This dataset does not include a feature matrix, so we use node2vec~\cite{grover2016node2vec} to extract two feature matrices as two views. IMDB is a movie network, where each node corresponds to a movie, and two adjacency matrices indicate whether two movies share the same actor or director. DBLP is a citation network, where each node corresponds to an academic research paper, and two adjacency matrices indicate whether two papers share the same authors or if one paper cites another. Amazon is a review network, where each node represents a product in the musical instruments category, and its attributes are extracted from product reviews. Similarly, the Yelp dataset contains hotel and restaurant reviews, either filtered (spam) or recommended (legitimate) by Yelp~\cite{DBLP:conf/cikm/DouL0DPY20}. The statistics of these graphs are summarized in Table~\ref{TB:Net}.

\subsubsection{Experiment Setting:}
The neural network structure of the proposed framework is GCN~\cite{DBLP:conf/iclr/KipfW17}. The hyper-parameters $\alpha$ and $\lambda$ for each dataset are specified in Appendix~\ref{reproducibility}. In all experiments, we set the initial learning rate to be 1e-5, the hidden feature dimension to be 128 and use Adam~\cite{DBLP:journals/corr/KingmaB14} as the optimizer. The similarity function $\text{sim}(a,b)$ is defined as $\text{sim}(a, b)=\exp(\frac{a \cdot b^T}{|a||b|})$. We use TAM as the backbone of our method to capture local node affinity. The number of GCN layers is set to 2. The experiments are performed on a Windows machine with a 24GB RTX 4090 GPU. 

\subsubsection{Evaluation Metrics:}
Following~\cite{DBLP:conf/nips/QiaoP23, DBLP:conf/kdd/PangHSC21}, all methods are evaluated based on Area Under the Receiver Operating Characteristic Curve (AUROC) and Area Under the Precision-Recall Curve (AUPRC). Higher AUROC/AUPRC indicates better performance. All of the experiments are repeated five times with different random seeds and the mean and standard deviation are reported. 

\subsubsection{Baseline Methods:}
In our experiments, we compare our proposed framework \method\ with state-of-the-art methods in the following two settings. 

For \textit{single-view graphs}, we compare \method\ with the following eight baseline methods: (1). \textbf{ANOMALOUS}~\cite{DBLP:conf/ijcai/PengLLLZ18}: a shallow method, jointly conducting attribute selection and anomaly detection as a whole based on CUR decomposition and residual analysis; (2). \textbf{Dominant}~\cite{DBLP:conf/sdm/DingLBL19}: a graph auto-encoder-based deep neural network model for graph anomaly detection, which encodes both the topological structure and node attributes to node embedding; (3). \textbf{CoLA}~\cite{DBLP:journals/tnn/LiuLPGZK22}: a contrastive self-supervised graph anomaly detection method by exploiting the local information; (4). \textbf{SLGAD}~\cite{zheng2021generative}: an unsupervised framework for outlier detection based on unlabeled in-distribution data, which uses contrastive learning loss as a regularization; (5). \textbf{HCM-A}~\cite{DBLP:conf/pkdd/HuangPMP22}: a self-supervised learning by forcing the prediction of the shortest path length between pairs of nodes; (6). \textbf{ComGA}~\cite{DBLP:conf/wsdm/Luo0BYZWX22}: a community-aware attributed graph anomaly detection framework; (7). \textbf{CONAD}~\cite{DBLP:conf/pakdd/XuHZDL22}: a contrastive learning-based graph anomaly detection method; (8). \textbf{TAM}~\cite{DBLP:conf/nips/QiaoP23}: an unsupervised anomaly method, proposing a scoring measure by assigning large score to the nodes that are less affiliated with their neighbors.

For \textit{multi-view graphs}, we compare \method\ with the following five baseline methods: (1). \textbf{MLRA}~\cite{DBLP:conf/sdm/LiSF15}: a multi-view non-negative matrix factorization-based method for anomaly detection, which performs cross-view low-rank analysis for revealing the intrinsic structures of data; (2). \textbf{NSNMF}~\cite{ahmed2021neighborhood}: an NMF based method, incorporating the neighborhood structural similarity information into the NMF framework to improve the anomaly detection performance; (3). \textbf{SRLSP}~\cite{DBLP:journals/tkdd/WangCLFZZ23}: a multi-view detection method based on the local similarity relation and data reconstruction; (4). \textbf{NCMOD}~\cite{DBLP:conf/aaai/Cheng0L21}: an auto-encoder-based multi-view anomaly detection method, which proposes neighborhood consensus networks to encode graph neighborhood information; (5) We also report the performance of \textbf{TAM}~\cite{DBLP:conf/nips/QiaoP23} on these multi-view graphs by averaging the adjacency matrix or concatenating two feature matrices together. We omit the comparison with other baseline methods since TAM outperforms them.

\begin{table*}
\caption{Results on multi-view graphs (\ie, CERT, IMDB, DBLP) with respect to AUROC and AUPRC. We boldface the best performance and underline the second-best.}
\centering
\vspace{-3mm}
\begin{tabular}{c|cc|cc|cc}
\hline \multirow{2}{*}{\textbf{Method}}  & \multicolumn{2}{c|}{\textbf{CERT}} & \multicolumn{2}{c|}{\textbf{IMDB}}  & \multicolumn{2}{c}{\textbf{DBLP}}  \\
 & AUPRC & AUROC & AUPRC & AUROC & AUPRC & AUROC \\ \hline
\hline MLRA    & 0.0379 $\pm$ 0.001  & 0.3829 $\pm$ 0.003  & 0.2695 $\pm$ 0.007  & 0.5926 $\pm$ 0.005 & 0.2211 $\pm$ 0.005  & 0.5568 $\pm$ 0.005 \\ 
NSNMF          & 0.0704 $\pm$ 0.001  & 0.4578 $\pm$ 0.001  & 0.0634 $\pm$ 0.000  & 0.4969 $\pm$ 0.001 & 0.1436 $\pm$ 0.007  & 0.6418 $\pm$ 0.001 \\
NCMOD          & 0.0749 $\pm$ 0.001  & 0.5133 $\pm$ 0.001 & \underline{0.6629 $\pm$ 0.013}  & 0.8030 $\pm$ 0.007 & \underline{0.4809 $\pm$ 0.006}  & \underline{0.7271 $\pm$ 0.004} \\
SRSLP	       & \underline{0.0806 $\pm$ 0.007}  & \underline{0.5405 $\pm$ 0.003} & 0.5552 $\pm$ 0.017  & 0.7343 $\pm$ 0.003 & 0.0643 $\pm$ 0.002  & 0.5228 $\pm$ 0.001 \\ 
TAM            & 0.0771 $\pm$ 0.007  & 0.5400 $\pm$ 0.005  & 0.6521 $\pm$ 0.016  & \underline{0.8233 $\pm$ 0.013} & 0.3466 $\pm$ 0.016  & 0.6690 $\pm$ 0.005 \\
\method        & \textbf{0.1198 $\pm$ 0.003}  & \textbf{0.6056 $\pm$ 0.001}  & \textbf{0.8968 $\pm$ 0.038} &	\textbf{0.9370 $\pm$ 0.029} &	\textbf{0.8495 $\pm$ 0.005} &	\textbf{0.8696 $\pm$ 0.006} \\
\hline
\end{tabular}
\label{table_multi_view_result}
\end{table*}

\begin{table*}
\caption{Results on single-view datasets (\ie, BlogCatalog, Amazon, YelpChi) with respect to AUROC and AUPRC. We boldface the best performance and underline the second-best.}
\centering
\vspace{-3mm}
\begin{tabular}{c|cc|cc|cc}
\hline \multirow{2}{*}{\textbf{Method}}  & \multicolumn{2}{c|}{\textbf{BlogCatalog}} & \multicolumn{2}{c|}{\textbf{Amazon}}  & \multicolumn{2}{c}{\textbf{YelpChi}}  \\
 & AUPRC & AUROC & AUPRC & AUROC & AUPRC & AUROC \\ \hline
\hline ANOMALOUS    & 0.0652 $\pm$ 0.005   & 0.5652 $\pm$ 0.025 & 0.0558 $\pm$ 0.001   & 0.4457 $\pm$ 0.005  & 0.0519 $\pm$ 0.002   & 0.4956 $\pm$ 0.003 \\ 
Dominant            & 0.3102 $\pm$ 0.011   & 0.7590 $\pm$ 0.010 & 0.1424 $\pm$ 0.002   & 0.5996 $\pm$ 0.002  & 0.0395 $\pm$ 0.020   & 0.4133 $\pm$ 0.100 \\
CoLA                & 0.3270 $\pm$ 0.000   & 0.7746 $\pm$ 0.009 & 0.0677 $\pm$ 0.001   & 0.5898 $\pm$ 0.011  & 0.0448 $\pm$ 0.002   & 0.4636 $\pm$ 0.001 \\
SLGAD               & 0.3882 $\pm$ 0.007   & 0.8123 $\pm$ 0.002 & 0.0634 $\pm$ 0.005   & 0.5937 $\pm$ 0.005  & 0.0350 $\pm$ 0.000   & 0.3312 $\pm$ 0.035 \\
HCM-A	            & 0.3139 $\pm$ 0.001   & 0.7980 $\pm$ 0.004 & 0.0527 $\pm$ 0.015   & 0.3956 $\pm$ 0.014  & 0.0287 $\pm$ 0.012   & 0.4593 $\pm$ 0.005 \\ 
ComGA               & 0.3293 $\pm$ 0.028   & 0.7683 $\pm$ 0.004 & 0.1153 $\pm$ 0.005   & 0.5895 $\pm$ 0.010  & 0.0423 $\pm$ 0.000   & 0.4391 $\pm$ 0.000 \\
CONAD	            & 0.3284 $\pm$ 0.004   & 0.7807 $\pm$ 0.003 & 0.1372 $\pm$ 0.009   & 0.6142 $\pm$ 0.008  & 0.0405 $\pm$ 0.002   & 0.4588 $\pm$ 0.003 \\
TAM                 & \textbf{0.4182 $\pm$ 0.225}   & \textbf{0.8248 $\pm$ 0.003} & \underline{0.2634 $\pm$ 0.008}   & \underline{0.7064 $\pm$ 0.008}  & \underline{0.0778 $\pm$ 0.009}   & \underline{0.5643 $\pm$ 0.007} \\
\method             & \underline{0.4043 $\pm$ 0.010}   & \underline{0.8194 $\pm$ 0.003} & \textbf{0.6563 $\pm$ 0.011}   & \textbf{0.8656 $\pm$ 0.002}  & \textbf{0.1218 $\pm$ 0.003}   & \textbf{0.7516 $\pm$ 0.003} \\
\hline
\end{tabular}
\label{table_single_view_result}
\end{table*}


\subsection{Experimental Analysis}

\subsubsection{Multi-view Graph Scenario.}
In this subsection, we evaluate the performance of \method\ by comparing it with five baseline methods on three multi-view graphs, \ie, CERT, IMDB and DBLP datasets. The experimental results with respect to AUROC and AUPRC are presented in Table~\ref{table_multi_view_result}. We observe the following: (i) MLRA and NSNMF exhibit the poor performance among all baseline methods across most datasets. This can be attributed to their design, which is optimized for independent and identically distributed data (e.g., image data). Thus, these two methods struggle to capture graph topological information, resulting in lower performance in both AUPRC and AUROC. (ii) In contrast, NCMOD incorporates graph topological information into the learned representation through its proposed neighborhood consensus networks, thus achieving the second-best performance on the IMDB and DBLP datasets. (iii) TAM suffers from the inability to handle multi-view graphs, performing worse than several self-supervised learning methods designed for multi-view graphs (e.g., NCMOD). (iv) Our proposed method outperforms all baselines in terms of AUROC and AUPRC across all datasets. Notably, graph-based anomaly detection methods designed for single views (e.g., TAM) falter in the presence of view heterogeneity as simply concatenating input features from multiple views distorts the feature space, resulting in significant performance challenges. In contrast, our method excels by encoding both local and global node affinities in the learned representation and mitigating biases through the proposed similarity-guided graph contrastive loss.


\subsubsection{Single-view Graph Scenario.} Next, we evaluate the performance of \method\ by comparing it to eight baseline methods on three single-view graphs, \ie, BlogCatalog, Amazon and YelpChi datasets. The experimental results with respect to AUROC and AUPRC are presented in Table~\ref{table_single_view_result}. We have the following observations: (i) TAM outperforms most baseline methods across three single-view graphs. We attribute its great performance to the design of normal structure-preserved graph truncation to remove edges connecting normal and abnormal nodes.  
(ii) While TAM excels on BlogCatalog, \method\ still ranks as a close second, showcasing its robustness in handling single-view graphs. (iii) \method\ proves to be a highly competitive approach, demonstrating superior performance on the Amazon and YelpChi datasets. Notably, \method\ surpasses TAM, the second-best method, by over 15.9\% in AUROC and 39.3\% in AUPRC on the Amazon dataset. Similarly, on the YelpChi dataset, \method\ improves AUROC by over 18.7\% compared to TAM. The ability of \method\ to capture both local and global node affinity information in the learned representations enables it to perform effectively on the Amazon and YelpChi datasets. By leveraging its similarity-guided graph contrastive loss, \method\ enhances representation learning, allowing it to better distinguish between normal and anomalous nodes. Overall, these results emphasize that while TAM leverages truncated affinity maximization techniques to tailor the raw adjacency matrix, \method\ offers a more flexible and powerful approach by incorporating global and local affinities, making it a highly effective method in the single-view graph anomaly detection scenario. Its ability to significantly outperform other baselines on challenging datasets like Amazon and YelpChi demonstrates the effectiveness of the proposed method.

\begin{table*}
\caption{Ablation study on BlogCatalog, Amazon, DBLP and IMDB datasets with respect to AUROC.}
\vspace{-3mm}
\centering
\begin{tabular}{*{6}{c}}
\hline \textbf{Model} & BlogCatalog   & Amazon  & DBLP & IMDB & AVERAGE\\ \hline
\hline \method              & \textbf{0.8194 $\pm$ 0.000} & \textbf{0.8656 $\pm$ 0.002} & \textbf{0.8868 $\pm$ 0.007} & \textbf{0.8792 $\pm$ 0.031} & \textbf{0.8628} \\
\hline \method-A            & 0.8144 $\pm$ 0.000 & 0.8514 $\pm$ 0.003 & 0.4831 $\pm$ 0.002 & 0.2138 $\pm$ 0.013 & 0.5907 \\ 
\hline \method-G            & 0.6313 $\pm$ 0.002 & 0.8645 $\pm$ 0.004 & 0.8639 $\pm$ 0.003 & 0.4996 $\pm$ 0.001 & 0.7148 \\
\hline \method-InfoNCE      & 0.7685 $\pm$ 0.001 & 0.7737 $\pm$ 0.006 & 0.8560 $\pm$ 0.005 & 0.8452 $\pm$ 0.032 & 0.8108 \\
\hline \method-WSC          & 0.7639 $\pm$ 0.011 & 0.8019 $\pm$ 0.004 & 0.8538 $\pm$ 0.006 & 0.8581 $\pm$ 0.005 & 0.8194 \\
\hline
\end{tabular}
\label{table_ablation_study}
\end{table*} 

\subsubsection{Ablation Study}
\label{Ablation_study}
In this subsection, we conduct an ablation study to demonstrate the necessity of each component of \method\ and validate the effectiveness of similarity-guided graph contrastive loss over vanilla and weakly supervised contrastive loss. Specifically, \method-A removes the similarity measurement of soft membership and replaces $\mathcal{L}_A$ with $\mathcal{L}_2$. \method-G refers to a variant of our proposed method by removing similarity-guided contrastive learning loss. \method-InfoNCE replaces the similarity-guide contrastive loss with the vanilla contrastive loss while \method-WSC substitutes it with the weakly supervised loss as shown in Eq.~\ref{L_1}. The experimental results with respect to AUROC on the BlogCatalog, Amazon, DBLP, and IMDB datasets are presented in Table~\ref{table_ablation_study}. Our observations are as follows: (1). \textbf{Global Node Affinity}: \method\ shows slight improvements on Amazon and BlogCatalog compared to \method-A. However, excluding the global node affinity matrix (\ie, the similarity map of the soft membership) in \method-A results in significant performance drops of approximately 40\% on DBLP and 56\% on IMDB, highlighting the importance of capturing global affinities in certain datasets.
(2). \textbf{Similarity-Guided Loss}: Removing the similarity-guided graph contrastive loss (i.e., \method-G) leads to an average AUROC performance drop of 15\% across the four datasets, underscoring the critical role of this regularization in mitigating bias and improving performance.
(3). \textbf{Contrastive Loss Substitutions}: Replacing the similarity-guided loss with either the vanilla contrastive loss (i.e., \method-InfoNCE) or the weakly supervised contrastive loss (i.e., \method-WSC) results in a 4\% to 5\% performance reduction on average. This aligns with the theoretical analysis in Theorem~\ref{theorem_1}, suggesting that neither the vanilla nor weakly supervised contrastive losses are as effective in mitigating bias as the similarity-guided approach.
These results validate the necessity of the global affinity information and the similarity-guided contrastive loss, both of which are essential for the robustness and effectiveness of \method.


\begin{figure*}[h]
\begin{center}
\begin{tabular}{cccc}
    \includegraphics[width=0.21\linewidth]{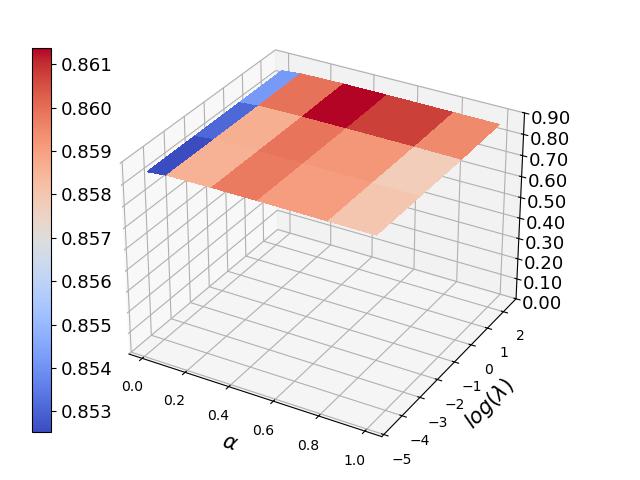} &
    \includegraphics[width=0.21\linewidth]{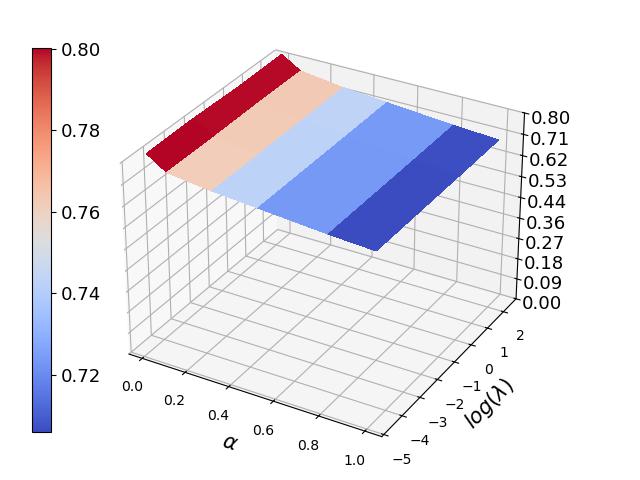} &
    \includegraphics[width=0.21\linewidth]{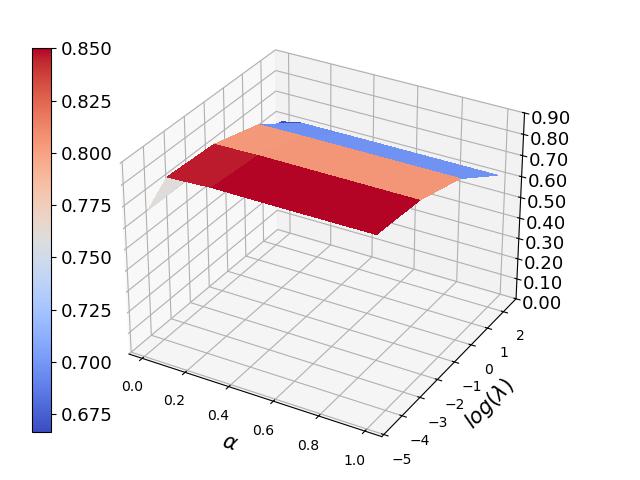} &
    \includegraphics[width=0.21\linewidth]{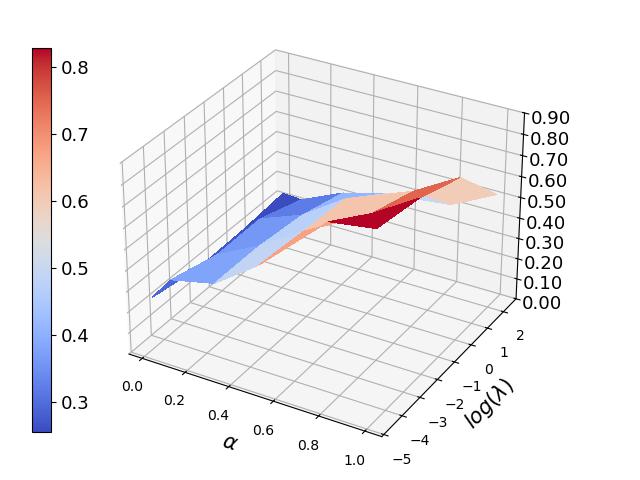}\\
    (a) Amazon dataset & (b) BlogCatalog dataset & (c) DBLP dataset & (d) IMDB dataset \\
\end{tabular}
\end{center}
\caption{$\alpha$, $\log(\lambda)$ v.s. AUROC on four datasets.}
\label{fig:parameter_analysis_alpha_and_lambda}
\end{figure*}

\begin{figure*}[h]
\begin{center}
\begin{tabular}{cccc}
    \includegraphics[width=0.21\linewidth]{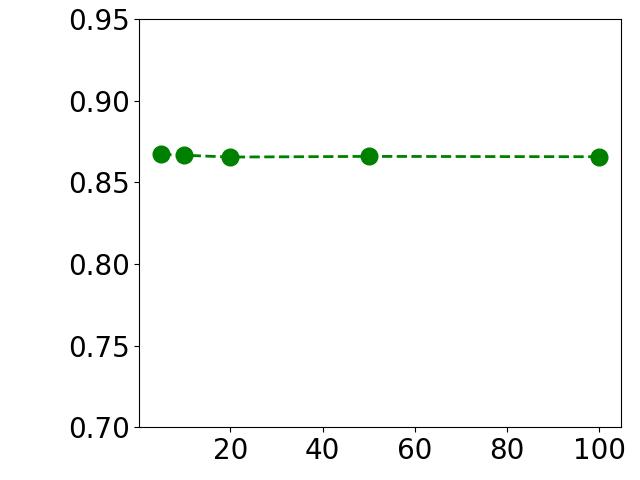} &
    \includegraphics[width=0.21\linewidth]{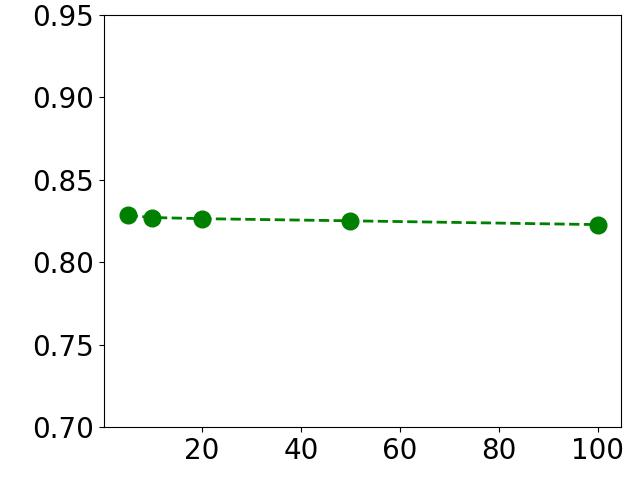} &
    \includegraphics[width=0.21\linewidth]{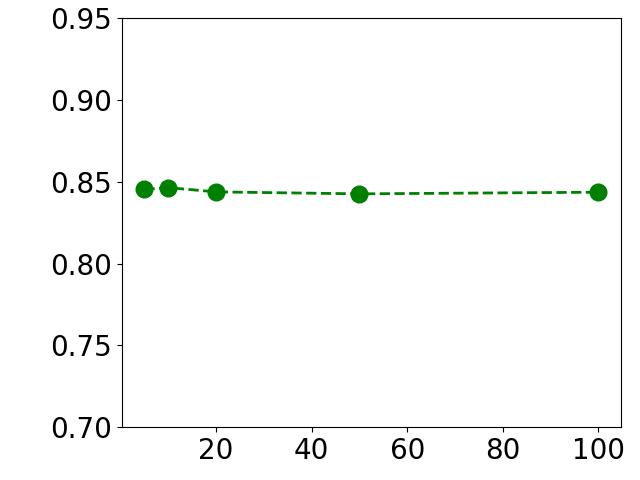} &
    \includegraphics[width=0.21\linewidth]{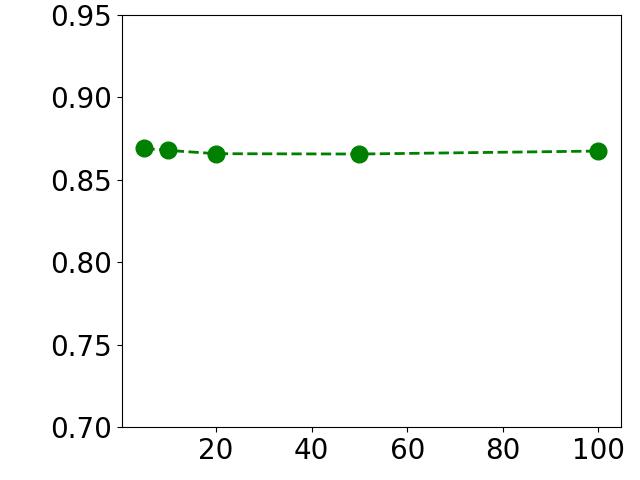}\\
    (a) Amazon dataset & (b) BlogCatalog dataset & (c) DBLP dataset & (d) IMDB dataset \\
\end{tabular}
\end{center}
\vspace{-2mm}
\caption{ The number of clusters v.s. AUROC on four datasets.}
\vspace{-2mm}
\label{fig:parameter_analysis}
\end{figure*}

\subsubsection{Parameter Analysis}
In this subsection, we delve into the parameter sensitivity analysis of the \method\ framework on the four datasets, specifically exploring the impact of $\alpha$, $\lambda$, and the number of clusters. In the experiment, the mean and standard deviation of AUROC over five rounds are reported. We fix the number of clusters to be 10 and vary the values of $\alpha$ and $\lambda$ while recording the AUROC of \method. The performance is visualized in Figure~\ref{fig:parameter_analysis_alpha_and_lambda}, where the x, y, and z axes represent $\alpha$, $\log(\lambda)$, and AUROC, respectively. From observation, \method\ achieves the best performance when $\alpha=0.4$ on the Amazon and DBLP datasets, while it reaches its peak performance at $\alpha=0.01$ on the BlogCatalog dataset and $\alpha=1$ on the IMDB dataset. Upon investigation, we find that the number of edges in the first view on the IMDB dataset is unusually large compared to the number of edges in the second view as presented in Table~\ref{TB:Net}. Due to this unusual pattern, we hypothesize that the raw adjacency matrix is not reliable and thus the global node affinity (\ie, similarity of soft assignment) plays a crucial role in detecting anomalous nodes. Thus, by setting $\alpha=1$, we exclude the local node affinity in the loss function $\mathcal{L}_A$ and therefore \method\ achieves better performance solely relying on the global node affinity. In terms of the parameter analysis on $\lambda$, we find that \method\ prefers a smaller value of $\lambda$ (\eg, $\lambda=0.1$). One explanation for this is that a large value of $\lambda$ tends to overly dominate the optimization of the overall objective function as $\lambda$ is used to balance the importance between the similarity-guided contrastive loss and the node affinity learning loss. In the subsequent experiment, we scrutinize the role of the number of clusters in shaping the anomaly detection criteria. Figure~\ref{fig:parameter_analysis} illustrates the results, with the x and y axes representing the values of the number of clusters and AUROC, respectively. We observe that changing the number of the clusters does not greatly influence the performance of \method\ across these four datasets. Notably, \method\ achieves better performance when the number of clusters is 5 or 10 on most datasets, suggesting that the proposed method prefers a small number of clusters. By reducing the number of clusters, our method tends to generalize better by capturing larger, more meaningful patterns in the data instead of overfitting to noise.

\subsubsection{Efficiency Analysis}
\label{efficiency_analysis}
\textbf{Time Complexity.} In this subsection, we analyze the time complexity of our proposed~\method. Assume that the graph has $n$ nodes, the input feature dimension is $d$, the hidden feature dimension is $f$, the number of nodes is $n$, the number of edges is $|E|$, and the number of clusters is $k$. For ease of explanation, we only consider the 1-layer case. Following \cite{blakely2021time}, the time complexity of computing the soft membership matrix using GCN is $O(ndk+n|E|k+n)$. The complexity of the GCN to capture the hidden representation $\bm{h}$ is $O(|E|d + ndf)$ with sparse computation. The complexity of the similarity-guided contrastive learning loss is $O(n^2f)$. However, in the experiments, we can use the sampling strategy to sample $p(p<<n)$ nodes and thus the complexity can be reduced to $O(p^2f)$. The total complexity of \method\ is $O(n(kd+fd+|E|k+1)+|E|d+p^2f)$.

\noindent\textbf{Running Time Analysis.} Next, we experiment on the YelpChi dataset to show the efficiency of our proposed~\method\ by changing the number of nodes and the number of layers. The reason why we only provide the efficiency analysis regarding the training/running time vs the number of nodes on the YelpChi dataset is that we can manually increase the number of nodes from 500 to 15,000 on this dataset. For the other datasets (e.g., IMDB, BlogCatalog, Amazon, DBLP), the number of nodes is less than 15,000 (see Table \ref{TB:Net} for details) and we cannot get the running time if the number of nodes is set to be a value larger than 15,000. In the first experiment, we fix the number of layers to be 1, the total number of iterations to be 10,000 and adjust the number of nodes by randomly selecting $k$ nodes, where $k\in[500, 1000, 2000, 5000, 10000, 15000]$. The experimental result is shown in Figure~\ref{fig:efficiency_analysis} (a). By observation, we find that the running time is quadratic to the number of nodes. In the second experiment, we fixed the number of nodes to be 2000, the total number of iterations to be 10,000, and adjusted the number of GCN layers from 1 layer to 4 layers. The experimental result is shown in Figure~\ref{fig:efficiency_analysis} (b). We observe that the running time is almost linear with respect to the number of layers.

\begin{figure}
\begin{center}
\begin{tabular}{cc}
    \includegraphics[width=0.40\linewidth]{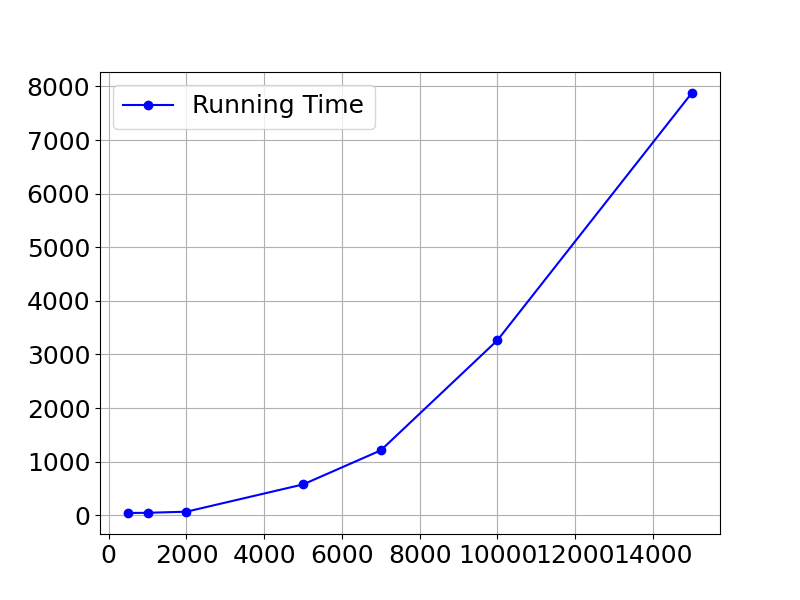} &
    \includegraphics[width=0.40\linewidth]{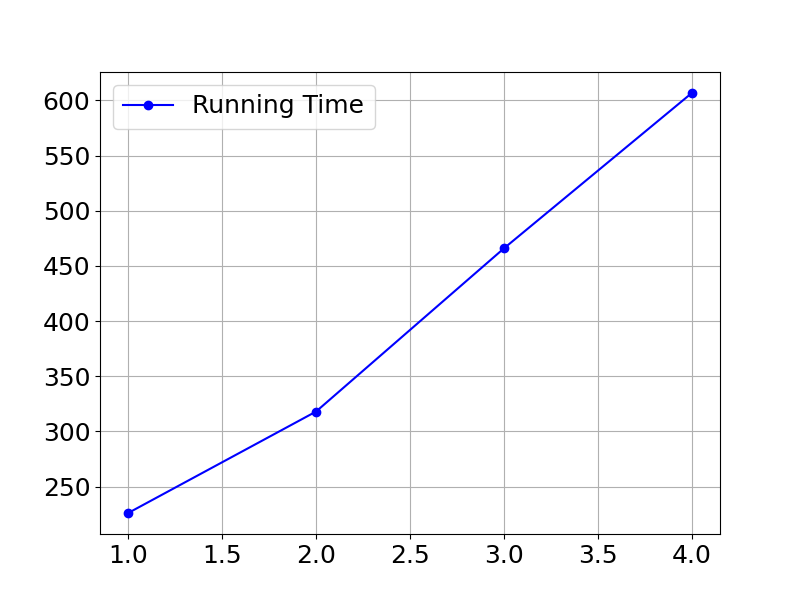} \\
    (a) The number of nodes vs &(b) The number of layers vs \\
    running time (in seconds) & running time (in seconds)\\
\end{tabular}
\end{center}
\vspace{-3mm}
\caption{Efficiency analysis on the YelpChi dataset}
\label{fig:efficiency_analysis}
\vspace{-3mm}
\end{figure}
\section{Conclusion}
In the past decades, graphs are ubiquitous and have been widely used in many real-world applications~\cite{ DBLP:conf/icde/Zheng0TXZH24, zheng2024pyg,  sun2021multi, sun2023all, DBLP:conf/cikm/ZhouZF0H22, DBLP:conf/kdd/ZhouZ0H20, DBLP:journals/fdata/ZhouZXH19, DBLP:conf/www/LiFH23, sun2021heterogeneous}. In this paper, we proposed a novel self-supervised framework for anomaly detection in multi-view graphs, addressing key limitations of existing methods. By capturing both local and global node affinities and leveraging a similarity-guided contrastive loss, our approach effectively identifies anomalous nodes without relying on strong structural assumptions. The proposed method not only augments the graph structure with learned soft membership assignments but also mitigates the negative impact of low-quality assignments through a robust loss function, theoretically connected to graph spectral clustering. Our experimental results on six datasets demonstrate the superior performance of the proposed framework in detecting anomalies in complex, multi-view graphs.

\begin{acks}
This work is supported by a research award from the C3.ai Digital Transformation Institute, and IBM-Illinois Discovery Accelerator Institute - a new model of an academic-industry partnership designed to increase access to technology education and skill development to spur breakthroughs in emerging areas of technology. Dr. John Birge was supported by the University of Chicago Booth School of Business. The views and conclusions are those of the authors and should not be interpreted as representing the official policies of the funding agencies or the government.
\end{acks}

\bibliographystyle{ACM-Reference-Format}
\balance
\bibliography{reference}


\begin{thebibliography}{91}


\ifx \showCODEN    \undefined \def \showCODEN     #1{\unskip}     \fi
\ifx \showISBNx    \undefined \def \showISBNx     #1{\unskip}     \fi
\ifx \showISBNxiii \undefined \def \showISBNxiii  #1{\unskip}     \fi
\ifx \showISSN     \undefined \def \showISSN      #1{\unskip}     \fi
\ifx \showLCCN     \undefined \def \showLCCN      #1{\unskip}     \fi
\ifx \shownote     \undefined \def \shownote      #1{#1}          \fi
\ifx \showarticletitle \undefined \def \showarticletitle #1{#1}   \fi
\ifx \showURL      \undefined \def \showURL       {\relax}        \fi
\providecommand\bibfield[2]{#2}
\providecommand\bibinfo[2]{#2}
\providecommand\natexlab[1]{#1}
\providecommand\showeprint[2][]{arXiv:#2}

\bibitem[Ahmed et~al\mbox{.}(2021)]%
        {ahmed2021neighborhood}
\bibfield{author}{\bibinfo{person}{Imtiaz Ahmed}, \bibinfo{person}{Xia~Ben Hu}, \bibinfo{person}{Mithun~P. Acharya}, {and} \bibinfo{person}{Yu Ding}.} \bibinfo{year}{2021}\natexlab{}.
\newblock \showarticletitle{Neighborhood Structure Assisted Non-negative Matrix Factorization and Its Application in Unsupervised Point-wise Anomaly Detection}.
\newblock \bibinfo{journal}{\emph{J. Mach. Learn. Res.}}  \bibinfo{volume}{22} (\bibinfo{year}{2021}), \bibinfo{pages}{34:1--34:32}.
\newblock


\bibitem[Akcay et~al\mbox{.}(2018)]%
        {DBLP:conf/accv/AkcayAB18}
\bibfield{author}{\bibinfo{person}{Samet Akcay}, \bibinfo{person}{Amir~Atapour Abarghouei}, {and} \bibinfo{person}{Toby~P. Breckon}.} \bibinfo{year}{2018}\natexlab{}.
\newblock \showarticletitle{GANomaly: Semi-supervised Anomaly Detection via Adversarial Training}. In \bibinfo{booktitle}{\emph{Computer Vision - {ACCV} 2018 - 14th Asian Conference on Computer Vision, Perth, Australia, December 2-6, 2018, Revised Selected Papers, Part {III}}}, Vol.~\bibinfo{volume}{11363}. \bibinfo{publisher}{Springer}, \bibinfo{pages}{622--637}.
\newblock


\bibitem[Allab et~al\mbox{.}(2017)]%
        {DBLP:journals/tkde/AllabLN17}
\bibfield{author}{\bibinfo{person}{Kais Allab}, \bibinfo{person}{Lazhar Labiod}, {and} \bibinfo{person}{Mohamed Nadif}.} \bibinfo{year}{2017}\natexlab{}.
\newblock \showarticletitle{A Semi-NMF-PCA Unified Framework for Data Clustering}.
\newblock \bibinfo{journal}{\emph{{IEEE} Trans. Knowl. Data Eng.}} \bibinfo{volume}{29}, \bibinfo{number}{1} (\bibinfo{year}{2017}), \bibinfo{pages}{2--16}.
\newblock


\bibitem[Bach and Jordan(2003)]%
        {DBLP:conf/nips/BachJ03}
\bibfield{author}{\bibinfo{person}{Francis~R. Bach} {and} \bibinfo{person}{Michael~I. Jordan}.} \bibinfo{year}{2003}\natexlab{}.
\newblock \showarticletitle{Learning Spectral Clustering}. In \bibinfo{booktitle}{\emph{Advances in Neural Information Processing Systems 16 [Neural Information Processing Systems, {NIPS} 2003, December 8-13, 2003, Vancouver and Whistler, British Columbia, Canada]}}. \bibinfo{publisher}{{MIT} Press}, \bibinfo{pages}{305--312}.
\newblock


\bibitem[Belkin and Niyogi(2003)]%
        {DBLP:journals/neco/BelkinN03}
\bibfield{author}{\bibinfo{person}{Mikhail Belkin} {and} \bibinfo{person}{Partha Niyogi}.} \bibinfo{year}{2003}\natexlab{}.
\newblock \showarticletitle{Laplacian Eigenmaps for Dimensionality Reduction and Data Representation}.
\newblock \bibinfo{journal}{\emph{Neural Comput.}} \bibinfo{volume}{15}, \bibinfo{number}{6} (\bibinfo{year}{2003}), \bibinfo{pages}{1373--1396}.
\newblock


\bibitem[Blakely et~al\mbox{.}(2021)]%
        {blakely2021time}
\bibfield{author}{\bibinfo{person}{Derrick Blakely}, \bibinfo{person}{Jack Lanchantin}, {and} \bibinfo{person}{Yanjun Qi}.} \bibinfo{year}{2021}\natexlab{}.
\newblock \showarticletitle{Time and space complexity of graph convolutional networks}.
\newblock \bibinfo{journal}{\emph{Accessed on: Dec}}  \bibinfo{volume}{31} (\bibinfo{year}{2021}).
\newblock


\bibitem[Bo et~al\mbox{.}(2020)]%
        {DBLP:conf/www/Bo0SZL020}
\bibfield{author}{\bibinfo{person}{Deyu Bo}, \bibinfo{person}{Xiao Wang}, \bibinfo{person}{Chuan Shi}, \bibinfo{person}{Meiqi Zhu}, \bibinfo{person}{Emiao Lu}, {and} \bibinfo{person}{Peng Cui}.} \bibinfo{year}{2020}\natexlab{}.
\newblock \showarticletitle{Structural Deep Clustering Network}. In \bibinfo{booktitle}{\emph{{WWW} '20: The Web Conference 2020, Taipei, Taiwan, April 20-24, 2020}}. \bibinfo{publisher}{{ACM} / {IW3C2}}, \bibinfo{pages}{1400--1410}.
\newblock


\bibitem[Brdiczka et~al\mbox{.}(2012)]%
        {brdiczka2012proactive}
\bibfield{author}{\bibinfo{person}{Oliver Brdiczka}, \bibinfo{person}{Juan Liu}, \bibinfo{person}{Bob Price}, \bibinfo{person}{Jianqiang Shen}, \bibinfo{person}{Akshay Patil}, \bibinfo{person}{Richard Chow}, \bibinfo{person}{Eugene Bart}, {and} \bibinfo{person}{Nicolas Ducheneaut}.} \bibinfo{year}{2012}\natexlab{}.
\newblock \showarticletitle{Proactive Insider Threat Detection through Graph Learning and Psychological Context}. In \bibinfo{booktitle}{\emph{2012 {IEEE} Symposium on Security and Privacy Workshops, San Francisco, CA, USA, May 24-25, 2012}}. \bibinfo{publisher}{{IEEE} Computer Society}, \bibinfo{pages}{142--149}.
\newblock


\bibitem[Chen et~al\mbox{.}(2023)]%
        {DBLP:journals/tkde/ChenZZDSZXKT23}
\bibfield{author}{\bibinfo{person}{Bo Chen}, \bibinfo{person}{Jing Zhang}, \bibinfo{person}{Xiaokang Zhang}, \bibinfo{person}{Yuxiao Dong}, \bibinfo{person}{Jian Song}, \bibinfo{person}{Peng Zhang}, \bibinfo{person}{Kaibo Xu}, \bibinfo{person}{Evgeny Kharlamov}, {and} \bibinfo{person}{Jie Tang}.} \bibinfo{year}{2023}\natexlab{}.
\newblock \showarticletitle{{GCCAD:} Graph Contrastive Coding for Anomaly Detection}.
\newblock \bibinfo{journal}{\emph{{IEEE} Trans. Knowl. Data Eng.}} \bibinfo{volume}{35}, \bibinfo{number}{8} (\bibinfo{year}{2023}), \bibinfo{pages}{8037--8051}.
\newblock


\bibitem[Chen et~al\mbox{.}(2016)]%
        {DBLP:conf/ijcai/ChenTSCZ16}
\bibfield{author}{\bibinfo{person}{Ting Chen}, \bibinfo{person}{Lu{-}An Tang}, \bibinfo{person}{Yizhou Sun}, \bibinfo{person}{Zhengzhang Chen}, {and} \bibinfo{person}{Kai Zhang}.} \bibinfo{year}{2016}\natexlab{}.
\newblock \showarticletitle{Entity Embedding-Based Anomaly Detection for Heterogeneous Categorical Events}. In \bibinfo{booktitle}{\emph{Proceedings of the Twenty-Fifth International Joint Conference on Artificial Intelligence, {IJCAI} 2016, New York, NY, USA, 9-15 July 2016}}. \bibinfo{publisher}{{IJCAI/AAAI} Press}, \bibinfo{pages}{1396--1403}.
\newblock


\bibitem[Cheng et~al\mbox{.}(2021)]%
        {DBLP:conf/aaai/Cheng0L21}
\bibfield{author}{\bibinfo{person}{Li Cheng}, \bibinfo{person}{Yijie Wang}, {and} \bibinfo{person}{Xinwang Liu}.} \bibinfo{year}{2021}\natexlab{}.
\newblock \showarticletitle{Neighborhood Consensus Networks for Unsupervised Multi-view Outlier Detection}. In \bibinfo{booktitle}{\emph{Thirty-Fifth {AAAI} Conference on Artificial Intelligence, {AAAI} 2021, Virtual Event, February 2-9, 2021}}. \bibinfo{publisher}{{AAAI} Press}, \bibinfo{pages}{7099--7106}.
\newblock


\bibitem[Ding et~al\mbox{.}(2019)]%
        {DBLP:conf/sdm/DingLBL19}
\bibfield{author}{\bibinfo{person}{Kaize Ding}, \bibinfo{person}{Jundong Li}, \bibinfo{person}{Rohit Bhanushali}, {and} \bibinfo{person}{Huan Liu}.} \bibinfo{year}{2019}\natexlab{}.
\newblock \showarticletitle{Deep Anomaly Detection on Attributed Networks}. In \bibinfo{booktitle}{\emph{Proceedings of the 2019 {SIAM} International Conference on Data Mining, {SDM} 2019, Calgary, Alberta, Canada, May 2-4, 2019}}. \bibinfo{publisher}{{SIAM}}, \bibinfo{pages}{594--602}.
\newblock


\bibitem[Dong et~al\mbox{.}(2024)]%
        {DBLP:journals/corr/abs-2405-17525}
\bibfield{author}{\bibinfo{person}{Xiangyu Dong}, \bibinfo{person}{Xingyi Zhang}, \bibinfo{person}{Yanni Sun}, \bibinfo{person}{Lei Chen}, \bibinfo{person}{Mingxuan Yuan}, {and} \bibinfo{person}{Sibo Wang}.} \bibinfo{year}{2024}\natexlab{}.
\newblock \showarticletitle{SmoothGNN: Smoothing-based {GNN} for Unsupervised Node Anomaly Detection}.
\newblock \bibinfo{journal}{\emph{CoRR}}  \bibinfo{volume}{abs/2405.17525} (\bibinfo{year}{2024}).
\newblock


\bibitem[Donoho et~al\mbox{.}(2000)]%
        {donoho2000high}
\bibfield{author}{\bibinfo{person}{David~L Donoho} {et~al\mbox{.}}} \bibinfo{year}{2000}\natexlab{}.
\newblock \showarticletitle{High-dimensional data analysis: The curses and blessings of dimensionality}.
\newblock \bibinfo{journal}{\emph{AMS math challenges lecture}} \bibinfo{volume}{1}, \bibinfo{number}{2000} (\bibinfo{year}{2000}), \bibinfo{pages}{32}.
\newblock


\bibitem[Dou et~al\mbox{.}(2020)]%
        {DBLP:conf/cikm/DouL0DPY20}
\bibfield{author}{\bibinfo{person}{Yingtong Dou}, \bibinfo{person}{Zhiwei Liu}, \bibinfo{person}{Li Sun}, \bibinfo{person}{Yutong Deng}, \bibinfo{person}{Hao Peng}, {and} \bibinfo{person}{Philip~S. Yu}.} \bibinfo{year}{2020}\natexlab{}.
\newblock \showarticletitle{Enhancing Graph Neural Network-based Fraud Detectors against Camouflaged Fraudsters}. In \bibinfo{booktitle}{\emph{{CIKM} '20: The 29th {ACM} International Conference on Information and Knowledge Management, Virtual Event, Ireland, October 19-23, 2020}}. \bibinfo{publisher}{{ACM}}, \bibinfo{pages}{315--324}.
\newblock


\bibitem[Duan et~al\mbox{.}(2023)]%
        {DBLP:conf/aaai/DuanW0ZHJ0D23}
\bibfield{author}{\bibinfo{person}{Jingcan Duan}, \bibinfo{person}{Siwei Wang}, \bibinfo{person}{Pei Zhang}, \bibinfo{person}{En Zhu}, \bibinfo{person}{Jingtao Hu}, \bibinfo{person}{Hu Jin}, \bibinfo{person}{Yue Liu}, {and} \bibinfo{person}{Zhibin Dong}.} \bibinfo{year}{2023}\natexlab{}.
\newblock \showarticletitle{Graph Anomaly Detection via Multi-Scale Contrastive Learning Networks with Augmented View}. In \bibinfo{booktitle}{\emph{Thirty-Seventh {AAAI} Conference on Artificial Intelligence, {AAAI} 2023, Washington, DC, USA, February 7-14, 2023}}. \bibinfo{publisher}{{AAAI} Press}, \bibinfo{pages}{7459--7467}.
\newblock


\bibitem[Eberle et~al\mbox{.}(2010)]%
        {eberle2010insider}
\bibfield{author}{\bibinfo{person}{William Eberle}, \bibinfo{person}{Jeffrey Graves}, {and} \bibinfo{person}{Lawrence Holder}.} \bibinfo{year}{2010}\natexlab{}.
\newblock \showarticletitle{Insider threat detection using a graph-based approach}.
\newblock \bibinfo{journal}{\emph{Journal of Applied Security Research}} \bibinfo{volume}{6}, \bibinfo{number}{1} (\bibinfo{year}{2010}), \bibinfo{pages}{32--81}.
\newblock


\bibitem[Glasser and Lindauer(2013)]%
        {DBLP:conf/sp/GlasserL13}
\bibfield{author}{\bibinfo{person}{Joshua Glasser} {and} \bibinfo{person}{Brian Lindauer}.} \bibinfo{year}{2013}\natexlab{}.
\newblock \showarticletitle{Bridging the Gap: {A} Pragmatic Approach to Generating Insider Threat Data}. In \bibinfo{booktitle}{\emph{2013 {IEEE} Symposium on Security and Privacy Workshops, San Francisco, CA, USA, May 23-24, 2013}}. \bibinfo{publisher}{{IEEE} Computer Society}, \bibinfo{pages}{98--104}.
\newblock


\bibitem[Grover and Leskovec(2016)]%
        {grover2016node2vec}
\bibfield{author}{\bibinfo{person}{Aditya Grover} {and} \bibinfo{person}{Jure Leskovec}.} \bibinfo{year}{2016}\natexlab{}.
\newblock \showarticletitle{node2vec: Scalable feature learning for networks}. In \bibinfo{booktitle}{\emph{Proceedings of the 22nd ACM SIGKDD international conference on Knowledge discovery and data mining}}. \bibinfo{pages}{855--864}.
\newblock


\bibitem[Grubbs(1969)]%
        {grubbs1969procedures}
\bibfield{author}{\bibinfo{person}{Frank~E Grubbs}.} \bibinfo{year}{1969}\natexlab{}.
\newblock \showarticletitle{Procedures for detecting outlying observations in samples}.
\newblock \bibinfo{journal}{\emph{Technometrics}} \bibinfo{volume}{11}, \bibinfo{number}{1} (\bibinfo{year}{1969}), \bibinfo{pages}{1--21}.
\newblock


\bibitem[Guan et~al\mbox{.}(2009)]%
        {guan2009fast}
\bibfield{author}{\bibinfo{person}{Xiaohong Guan}, \bibinfo{person}{Wei Wang}, {and} \bibinfo{person}{Xiangliang Zhang}.} \bibinfo{year}{2009}\natexlab{}.
\newblock \showarticletitle{Fast intrusion detection based on a non-negative matrix factorization model}.
\newblock \bibinfo{journal}{\emph{J. Netw. Comput. Appl.}} \bibinfo{volume}{32}, \bibinfo{number}{1} (\bibinfo{year}{2009}), \bibinfo{pages}{31--44}.
\newblock


\bibitem[Guo et~al\mbox{.}(2017)]%
        {DBLP:conf/iconip/GuoLZY17}
\bibfield{author}{\bibinfo{person}{Xifeng Guo}, \bibinfo{person}{Xinwang Liu}, \bibinfo{person}{En Zhu}, {and} \bibinfo{person}{Jianping Yin}.} \bibinfo{year}{2017}\natexlab{}.
\newblock \showarticletitle{Deep Clustering with Convolutional Autoencoders}. In \bibinfo{booktitle}{\emph{Neural Information Processing - 24th International Conference, {ICONIP} 2017, Guangzhou, China, November 14-18, 2017, Proceedings, Part {II}}}, Vol.~\bibinfo{volume}{10635}. \bibinfo{publisher}{Springer}, \bibinfo{pages}{373--382}.
\newblock


\bibitem[HaoChen et~al\mbox{.}(2021)]%
        {DBLP:conf/nips/HaoChenWGM21}
\bibfield{author}{\bibinfo{person}{Jeff~Z. HaoChen}, \bibinfo{person}{Colin Wei}, \bibinfo{person}{Adrien Gaidon}, {and} \bibinfo{person}{Tengyu Ma}.} \bibinfo{year}{2021}\natexlab{}.
\newblock \showarticletitle{Provable Guarantees for Self-Supervised Deep Learning with Spectral Contrastive Loss}. In \bibinfo{booktitle}{\emph{Advances in Neural Information Processing Systems 34: Annual Conference on Neural Information Processing Systems 2021, NeurIPS 2021, December 6-14, 2021, virtual}}. \bibinfo{pages}{5000--5011}.
\newblock


\bibitem[Hooi et~al\mbox{.}(2016)]%
        {DBLP:conf/kdd/HooiSBSSF16}
\bibfield{author}{\bibinfo{person}{Bryan Hooi}, \bibinfo{person}{Hyun~Ah Song}, \bibinfo{person}{Alex Beutel}, \bibinfo{person}{Neil Shah}, \bibinfo{person}{Kijung Shin}, {and} \bibinfo{person}{Christos Faloutsos}.} \bibinfo{year}{2016}\natexlab{}.
\newblock \showarticletitle{{FRAUDAR:} Bounding Graph Fraud in the Face of Camouflage}. In \bibinfo{booktitle}{\emph{Proceedings of the 22nd {ACM} {SIGKDD} International Conference on Knowledge Discovery and Data Mining, San Francisco, CA, USA, August 13-17, 2016}}. \bibinfo{publisher}{{ACM}}, \bibinfo{pages}{895--904}.
\newblock


\bibitem[Huang et~al\mbox{.}(2022)]%
        {DBLP:conf/pkdd/HuangPMP22}
\bibfield{author}{\bibinfo{person}{Tianjin Huang}, \bibinfo{person}{Yulong Pei}, \bibinfo{person}{Vlado Menkovski}, {and} \bibinfo{person}{Mykola Pechenizkiy}.} \bibinfo{year}{2022}\natexlab{}.
\newblock \showarticletitle{Hop-Count Based Self-supervised Anomaly Detection on Attributed Networks}. In \bibinfo{booktitle}{\emph{Machine Learning and Knowledge Discovery in Databases - European Conference, {ECML} {PKDD} 2022, Grenoble, France, September 19-23, 2022, Proceedings, Part {I}}} \emph{(\bibinfo{series}{Lecture Notes in Computer Science}, Vol.~\bibinfo{volume}{13713})}. \bibinfo{publisher}{Springer}, \bibinfo{pages}{225--241}.
\newblock


\bibitem[Jiang et~al\mbox{.}(2019)]%
        {jiang2019anomaly}
\bibfield{author}{\bibinfo{person}{Jianguo Jiang}, \bibinfo{person}{Jiuming Chen}, \bibinfo{person}{Tianbo Gu}, \bibinfo{person}{Kim{-}Kwang~Raymond Choo}, \bibinfo{person}{Chao Liu}, \bibinfo{person}{Min Yu}, \bibinfo{person}{Weiqing Huang}, {and} \bibinfo{person}{Prasant Mohapatra}.} \bibinfo{year}{2019}\natexlab{}.
\newblock \showarticletitle{Anomaly Detection with Graph Convolutional Networks for Insider Threat and Fraud Detection}. In \bibinfo{booktitle}{\emph{2019 {IEEE} Military Communications Conference, {MILCOM} 2019, Norfolk, VA, USA, November 12-14, 2019}}. \bibinfo{publisher}{{IEEE}}, \bibinfo{pages}{109--114}.
\newblock


\bibitem[J{\'\i}lkov{\'a} and Kr{\'a}lov{\'a}(2021)]%
        {jilkova2021digital}
\bibfield{author}{\bibinfo{person}{Petra J{\'\i}lkov{\'a}} {and} \bibinfo{person}{Petra Kr{\'a}lov{\'a}}.} \bibinfo{year}{2021}\natexlab{}.
\newblock \showarticletitle{Digital consumer behaviour and ecommerce trends during the COVID-19 crisis}.
\newblock \bibinfo{journal}{\emph{International Advances in Economic Research}} \bibinfo{volume}{27}, \bibinfo{number}{1} (\bibinfo{year}{2021}), \bibinfo{pages}{83--85}.
\newblock


\bibitem[Jin et~al\mbox{.}(2021)]%
        {DBLP:conf/cikm/JinLZCLP21}
\bibfield{author}{\bibinfo{person}{Ming Jin}, \bibinfo{person}{Yixin Liu}, \bibinfo{person}{Yu Zheng}, \bibinfo{person}{Lianhua Chi}, \bibinfo{person}{Yuan{-}Fang Li}, {and} \bibinfo{person}{Shirui Pan}.} \bibinfo{year}{2021}\natexlab{}.
\newblock \showarticletitle{{ANEMONE:} Graph Anomaly Detection with Multi-Scale Contrastive Learning}. In \bibinfo{booktitle}{\emph{{CIKM} '21: The 30th {ACM} International Conference on Information and Knowledge Management, Virtual Event, Queensland, Australia, November 1 - 5, 2021}}. \bibinfo{publisher}{{ACM}}, \bibinfo{pages}{3122--3126}.
\newblock


\bibitem[Jing et~al\mbox{.}(2022)]%
        {jing2022x}
\bibfield{author}{\bibinfo{person}{Baoyu Jing}, \bibinfo{person}{Shengyu Feng}, \bibinfo{person}{Yuejia Xiang}, \bibinfo{person}{Xi Chen}, \bibinfo{person}{Yu Chen}, {and} \bibinfo{person}{Hanghang Tong}.} \bibinfo{year}{2022}\natexlab{}.
\newblock \showarticletitle{X-GOAL: Multiplex heterogeneous graph prototypical contrastive learning}. In \bibinfo{booktitle}{\emph{Proceedings of the 31st ACM International Conference on Information \& Knowledge Management}}. \bibinfo{pages}{894--904}.
\newblock


\bibitem[Jing et~al\mbox{.}(2021)]%
        {jing2021hdmi}
\bibfield{author}{\bibinfo{person}{Baoyu Jing}, \bibinfo{person}{Chanyoung Park}, {and} \bibinfo{person}{Hanghang Tong}.} \bibinfo{year}{2021}\natexlab{}.
\newblock \showarticletitle{Hdmi: High-order deep multiplex infomax}. In \bibinfo{booktitle}{\emph{Proceedings of the Web Conference 2021}}. \bibinfo{pages}{2414--2424}.
\newblock


\bibitem[Jing et~al\mbox{.}(2024a)]%
        {jing2024automated}
\bibfield{author}{\bibinfo{person}{Baoyu Jing}, \bibinfo{person}{Yansen Wang}, \bibinfo{person}{Guoxin Sui}, \bibinfo{person}{Jing Hong}, \bibinfo{person}{Jingrui He}, \bibinfo{person}{Yuqing Yang}, \bibinfo{person}{Dongsheng Li}, {and} \bibinfo{person}{Kan Ren}.} \bibinfo{year}{2024}\natexlab{a}.
\newblock \showarticletitle{Automated contrastive learning strategy search for time series}. In \bibinfo{booktitle}{\emph{Proceedings of the 33rd ACM International Conference on Information and Knowledge Management}}. \bibinfo{pages}{4612--4620}.
\newblock


\bibitem[Jing et~al\mbox{.}(2024b)]%
        {jing2024sterling}
\bibfield{author}{\bibinfo{person}{Baoyu Jing}, \bibinfo{person}{Yuchen Yan}, \bibinfo{person}{Kaize Ding}, \bibinfo{person}{Chanyoung Park}, \bibinfo{person}{Yada Zhu}, \bibinfo{person}{Huan Liu}, {and} \bibinfo{person}{Hanghang Tong}.} \bibinfo{year}{2024}\natexlab{b}.
\newblock \showarticletitle{Sterling: Synergistic representation learning on bipartite graphs}. In \bibinfo{booktitle}{\emph{Proceedings of the AAAI Conference on Artificial Intelligence}}, Vol.~\bibinfo{volume}{38}. \bibinfo{pages}{12976--12984}.
\newblock


\bibitem[Jing et~al\mbox{.}({[n.\,d.]})]%
        {jing2022coin}
\bibfield{author}{\bibinfo{person}{Baoyu Jing}, \bibinfo{person}{Yuchen Yan}, \bibinfo{person}{Yada Zhu}, {and} \bibinfo{person}{Hanghang Tong}.} \bibinfo{year}{[n.\,d.]}\natexlab{}.
\newblock \showarticletitle{Coin: Co-cluster infomax for bipartite graphs}. In \bibinfo{booktitle}{\emph{NeurIPS 2022 Workshop: New Frontiers in Graph Learning}}.
\newblock


\bibitem[Kingma and Ba(2015)]%
        {DBLP:journals/corr/KingmaB14}
\bibfield{author}{\bibinfo{person}{Diederik~P. Kingma} {and} \bibinfo{person}{Jimmy Ba}.} \bibinfo{year}{2015}\natexlab{}.
\newblock \showarticletitle{Adam: {A} Method for Stochastic Optimization}. In \bibinfo{booktitle}{\emph{3rd International Conference on Learning Representations, {ICLR} 2015, San Diego, CA, USA, May 7-9, 2015, Conference Track Proceedings}}.
\newblock


\bibitem[Kipf and Welling(2017)]%
        {DBLP:conf/iclr/KipfW17}
\bibfield{author}{\bibinfo{person}{Thomas~N. Kipf} {and} \bibinfo{person}{Max Welling}.} \bibinfo{year}{2017}\natexlab{}.
\newblock \showarticletitle{Semi-Supervised Classification with Graph Convolutional Networks}. In \bibinfo{booktitle}{\emph{5th International Conference on Learning Representations, {ICLR} 2017, Toulon, France, April 24-26, 2017, Conference Track Proceedings}}. \bibinfo{publisher}{OpenReview.net}.
\newblock


\bibitem[Kumar et~al\mbox{.}(2019)]%
        {DBLP:conf/kdd/KumarZL19}
\bibfield{author}{\bibinfo{person}{Srijan Kumar}, \bibinfo{person}{Xikun Zhang}, {and} \bibinfo{person}{Jure Leskovec}.} \bibinfo{year}{2019}\natexlab{}.
\newblock \showarticletitle{Predicting Dynamic Embedding Trajectory in Temporal Interaction Networks}. In \bibinfo{booktitle}{\emph{Proceedings of the 25th {ACM} {SIGKDD} International Conference on Knowledge Discovery {\&} Data Mining, {KDD} 2019, Anchorage, AK, USA, August 4-8, 2019}}. \bibinfo{publisher}{{ACM}}, \bibinfo{pages}{1269--1278}.
\newblock


\bibitem[Lan et~al\mbox{.}(2024)]%
        {DBLP:conf/emnlp/LanZMK24}
\bibfield{author}{\bibinfo{person}{Mengfei Lan}, \bibinfo{person}{Lecheng Zheng}, \bibinfo{person}{Shufan Ming}, {and} \bibinfo{person}{Halil Kilicoglu}.} \bibinfo{year}{2024}\natexlab{}.
\newblock \showarticletitle{Multi-label Sequential Sentence Classification via Large Language Model}. In \bibinfo{booktitle}{\emph{Findings of the Association for Computational Linguistics: {EMNLP} 2024, Miami, Florida, USA, November 12-16, 2024}}. \bibinfo{publisher}{Association for Computational Linguistics}, \bibinfo{pages}{16086--16104}.
\newblock


\bibitem[Li et~al\mbox{.}(2021a)]%
        {DBLP:conf/cvpr/LiSYP21}
\bibfield{author}{\bibinfo{person}{Chun{-}Liang Li}, \bibinfo{person}{Kihyuk Sohn}, \bibinfo{person}{Jinsung Yoon}, {and} \bibinfo{person}{Tomas Pfister}.} \bibinfo{year}{2021}\natexlab{a}.
\newblock \showarticletitle{CutPaste: Self-Supervised Learning for Anomaly Detection and Localization}. In \bibinfo{booktitle}{\emph{{IEEE} Conference on Computer Vision and Pattern Recognition, {CVPR} 2021, June 19-25, 2021}}. \bibinfo{publisher}{Computer Vision Foundation / {IEEE}}, \bibinfo{pages}{9664--9674}.
\newblock


\bibitem[Li et~al\mbox{.}(2021b)]%
        {DBLP:conf/aaai/LiZZH21}
\bibfield{author}{\bibinfo{person}{Jianbo Li}, \bibinfo{person}{Lecheng Zheng}, \bibinfo{person}{Yada Zhu}, {and} \bibinfo{person}{Jingrui He}.} \bibinfo{year}{2021}\natexlab{b}.
\newblock \showarticletitle{Outlier Impact Characterization for Time Series Data}. In \bibinfo{booktitle}{\emph{Thirty-Fifth {AAAI} Conference on Artificial Intelligence, {AAAI} 2021, Virtual Event, February 2-9, 2021}}. \bibinfo{publisher}{{AAAI} Press}, \bibinfo{pages}{11595--11603}.
\newblock


\bibitem[Li et~al\mbox{.}(2015)]%
        {DBLP:conf/sdm/LiSF15}
\bibfield{author}{\bibinfo{person}{Sheng Li}, \bibinfo{person}{Ming Shao}, {and} \bibinfo{person}{Yun Fu}.} \bibinfo{year}{2015}\natexlab{}.
\newblock \showarticletitle{Multi-View Low-Rank Analysis for Outlier Detection}. In \bibinfo{booktitle}{\emph{Proceedings of the 2015 {SIAM} International Conference on Data Mining, Vancouver, BC, Canada, April 30 - May 2, 2015}}. \bibinfo{publisher}{{SIAM}}, \bibinfo{pages}{748--756}.
\newblock


\bibitem[Li et~al\mbox{.}(2023)]%
        {DBLP:conf/www/LiFH23}
\bibfield{author}{\bibinfo{person}{Zihao Li}, \bibinfo{person}{Dongqi Fu}, {and} \bibinfo{person}{Jingrui He}.} \bibinfo{year}{2023}\natexlab{}.
\newblock \showarticletitle{Everything Evolves in Personalized PageRank}. In \bibinfo{booktitle}{\emph{Proceedings of the {ACM} Web Conference 2023, {WWW} 2023, Austin, TX, USA, 30 April 2023 - 4 May 2023}}. \bibinfo{publisher}{{ACM}}, \bibinfo{pages}{3342--3352}.
\newblock


\bibitem[Li and Liu(2009)]%
        {DBLP:conf/iccv/LiL09}
\bibfield{author}{\bibinfo{person}{Zhenguo Li} {and} \bibinfo{person}{Jianzhuang Liu}.} \bibinfo{year}{2009}\natexlab{}.
\newblock \showarticletitle{Constrained clustering by spectral kernel learning}. In \bibinfo{booktitle}{\emph{{IEEE} 12th International Conference on Computer Vision, {ICCV} 2009, Kyoto, Japan, September 27 - October 4, 2009}}. \bibinfo{publisher}{{IEEE} Computer Society}, \bibinfo{pages}{421--427}.
\newblock


\bibitem[Li et~al\mbox{.}(2024)]%
        {li2024can}
\bibfield{author}{\bibinfo{person}{Zihao Li}, \bibinfo{person}{Lecheng Zheng}, \bibinfo{person}{Bowen Jin}, \bibinfo{person}{Dongqi Fu}, \bibinfo{person}{Baoyu Jing}, \bibinfo{person}{Yikun Ban}, \bibinfo{person}{Jingrui He}, {and} \bibinfo{person}{Jiawei Han}.} \bibinfo{year}{2024}\natexlab{}.
\newblock \showarticletitle{Can Graph Neural Networks Learn Language with Extremely Weak Text Supervision?}
\newblock \bibinfo{journal}{\emph{arXiv preprint arXiv:2412.08174}} (\bibinfo{year}{2024}).
\newblock


\bibitem[Liu et~al\mbox{.}(2022a)]%
        {DBLP:journals/tnn/LiuLPGZK22}
\bibfield{author}{\bibinfo{person}{Yixin Liu}, \bibinfo{person}{Zhao Li}, \bibinfo{person}{Shirui Pan}, \bibinfo{person}{Chen Gong}, \bibinfo{person}{Chuan Zhou}, {and} \bibinfo{person}{George Karypis}.} \bibinfo{year}{2022}\natexlab{a}.
\newblock \showarticletitle{Anomaly Detection on Attributed Networks via Contrastive Self-Supervised Learning}.
\newblock \bibinfo{journal}{\emph{{IEEE} Trans. Neural Networks Learn. Syst.}} \bibinfo{volume}{33}, \bibinfo{number}{6} (\bibinfo{year}{2022}), \bibinfo{pages}{2378--2392}.
\newblock


\bibitem[Liu et~al\mbox{.}(2022b)]%
        {DBLP:conf/aaai/LiuTZLSYZ22}
\bibfield{author}{\bibinfo{person}{Yue Liu}, \bibinfo{person}{Wenxuan Tu}, \bibinfo{person}{Sihang Zhou}, \bibinfo{person}{Xinwang Liu}, \bibinfo{person}{Linxuan Song}, \bibinfo{person}{Xihong Yang}, {and} \bibinfo{person}{En Zhu}.} \bibinfo{year}{2022}\natexlab{b}.
\newblock \showarticletitle{Deep Graph Clustering via Dual Correlation Reduction}. In \bibinfo{booktitle}{\emph{Thirty-Sixth {AAAI} Conference on Artificial Intelligence, {AAAI} 2022, Virtual Event, February 22 - March 1, 2022}}. \bibinfo{publisher}{{AAAI} Press}, \bibinfo{pages}{7603--7611}.
\newblock


\bibitem[Luo et~al\mbox{.}(2022)]%
        {DBLP:conf/wsdm/Luo0BYZWX22}
\bibfield{author}{\bibinfo{person}{Xuexiong Luo}, \bibinfo{person}{Jia Wu}, \bibinfo{person}{Amin Beheshti}, \bibinfo{person}{Jian Yang}, \bibinfo{person}{Xiankun Zhang}, \bibinfo{person}{Yuan Wang}, {and} \bibinfo{person}{Shan Xue}.} \bibinfo{year}{2022}\natexlab{}.
\newblock \showarticletitle{ComGA: Community-Aware Attributed Graph Anomaly Detection}. In \bibinfo{booktitle}{\emph{{WSDM} '22: The Fifteenth {ACM} International Conference on Web Search and Data Mining, Virtual Event / Tempe, AZ, USA, February 21 - 25, 2022}}. \bibinfo{publisher}{{ACM}}, \bibinfo{pages}{657--665}.
\newblock


\bibitem[Ma et~al\mbox{.}(2021)]%
        {ma2021comprehensive}
\bibfield{author}{\bibinfo{person}{Xiaoxiao Ma}, \bibinfo{person}{Jia Wu}, \bibinfo{person}{Shan Xue}, \bibinfo{person}{Jian Yang}, \bibinfo{person}{Chuan Zhou}, \bibinfo{person}{Quan~Z Sheng}, \bibinfo{person}{Hui Xiong}, {and} \bibinfo{person}{Leman Akoglu}.} \bibinfo{year}{2021}\natexlab{}.
\newblock \showarticletitle{A comprehensive survey on graph anomaly detection with deep learning}.
\newblock \bibinfo{journal}{\emph{IEEE Transactions on Knowledge and Data Engineering}} (\bibinfo{year}{2021}).
\newblock


\bibitem[Ng et~al\mbox{.}(2001)]%
        {DBLP:conf/nips/NgJW01}
\bibfield{author}{\bibinfo{person}{Andrew~Y. Ng}, \bibinfo{person}{Michael~I. Jordan}, {and} \bibinfo{person}{Yair Weiss}.} \bibinfo{year}{2001}\natexlab{}.
\newblock \showarticletitle{On Spectral Clustering: Analysis and an algorithm}. In \bibinfo{booktitle}{\emph{Advances in Neural Information Processing Systems 14 [Neural Information Processing Systems: Natural and Synthetic, {NIPS} 2001, December 3-8, 2001, Vancouver, British Columbia, Canada]}}. \bibinfo{publisher}{{MIT} Press}, \bibinfo{pages}{849--856}.
\newblock


\bibitem[Pang et~al\mbox{.}(2021)]%
        {DBLP:conf/kdd/PangHSC21}
\bibfield{author}{\bibinfo{person}{Guansong Pang}, \bibinfo{person}{Anton van~den Hengel}, \bibinfo{person}{Chunhua Shen}, {and} \bibinfo{person}{Longbing Cao}.} \bibinfo{year}{2021}\natexlab{}.
\newblock \showarticletitle{Toward Deep Supervised Anomaly Detection: Reinforcement Learning from Partially Labeled Anomaly Data}. In \bibinfo{booktitle}{\emph{{KDD} '21: The 27th {ACM} {SIGKDD} Conference on Knowledge Discovery and Data Mining, Virtual Event, Singapore, August 14-18, 2021}}. \bibinfo{publisher}{{ACM}}, \bibinfo{pages}{1298--1308}.
\newblock


\bibitem[Pang et~al\mbox{.}(2020)]%
        {DBLP:conf/cvpr/PangYSH020}
\bibfield{author}{\bibinfo{person}{Guansong Pang}, \bibinfo{person}{Cheng Yan}, \bibinfo{person}{Chunhua Shen}, \bibinfo{person}{Anton van~den Hengel}, {and} \bibinfo{person}{Xiao Bai}.} \bibinfo{year}{2020}\natexlab{}.
\newblock \showarticletitle{Self-Trained Deep Ordinal Regression for End-to-End Video Anomaly Detection}. In \bibinfo{booktitle}{\emph{2020 {IEEE/CVF} Conference on Computer Vision and Pattern Recognition, {CVPR} 2020, Seattle, WA, USA, June 13-19, 2020}}. \bibinfo{publisher}{Computer Vision Foundation / {IEEE}}, \bibinfo{pages}{12170--12179}.
\newblock


\bibitem[Park et~al\mbox{.}(2020)]%
        {DBLP:conf/cvpr/ParkNH20}
\bibfield{author}{\bibinfo{person}{Hyunjong Park}, \bibinfo{person}{Jongyoun Noh}, {and} \bibinfo{person}{Bumsub Ham}.} \bibinfo{year}{2020}\natexlab{}.
\newblock \showarticletitle{Learning Memory-Guided Normality for Anomaly Detection}. In \bibinfo{booktitle}{\emph{2020 {IEEE/CVF} Conference on Computer Vision and Pattern Recognition, {CVPR} 2020, Seattle, WA, USA, June 13-19, 2020}}. \bibinfo{publisher}{Computer Vision Foundation / {IEEE}}, \bibinfo{pages}{14360--14369}.
\newblock


\bibitem[Peng et~al\mbox{.}(2018)]%
        {DBLP:conf/ijcai/PengLLLZ18}
\bibfield{author}{\bibinfo{person}{Zhen Peng}, \bibinfo{person}{Minnan Luo}, \bibinfo{person}{Jundong Li}, \bibinfo{person}{Huan Liu}, {and} \bibinfo{person}{Qinghua Zheng}.} \bibinfo{year}{2018}\natexlab{}.
\newblock \showarticletitle{{ANOMALOUS:} {A} Joint Modeling Approach for Anomaly Detection on Attributed Networks}. In \bibinfo{booktitle}{\emph{Proceedings of the Twenty-Seventh International Joint Conference on Artificial Intelligence, {IJCAI} 2018, July 13-19, 2018, Stockholm, Sweden}}. \bibinfo{publisher}{ijcai.org}, \bibinfo{pages}{3513--3519}.
\newblock


\bibitem[Pourhabibi et~al\mbox{.}(2020)]%
        {pourhabibi2020fraud}
\bibfield{author}{\bibinfo{person}{Tahereh Pourhabibi}, \bibinfo{person}{Kok-Leong Ong}, \bibinfo{person}{Booi~H Kam}, {and} \bibinfo{person}{Yee~Ling Boo}.} \bibinfo{year}{2020}\natexlab{}.
\newblock \showarticletitle{Fraud detection: A systematic literature review of graph-based anomaly detection approaches}.
\newblock \bibinfo{journal}{\emph{Decision Support Systems}}  \bibinfo{volume}{133} (\bibinfo{year}{2020}), \bibinfo{pages}{113303}.
\newblock


\bibitem[Qi et~al\mbox{.}(2022)]%
        {DBLP:conf/kdd/QiBH22}
\bibfield{author}{\bibinfo{person}{Yunzhe Qi}, \bibinfo{person}{Yikun Ban}, {and} \bibinfo{person}{Jingrui He}.} \bibinfo{year}{2022}\natexlab{}.
\newblock \showarticletitle{Neural Bandit with Arm Group Graph}. In \bibinfo{booktitle}{\emph{{KDD} '22: The 28th {ACM} {SIGKDD} Conference on Knowledge Discovery and Data Mining, Washington, DC, USA, August 14 - 18, 2022}}. \bibinfo{publisher}{{ACM}}, \bibinfo{pages}{1379--1389}.
\newblock


\bibitem[Qi et~al\mbox{.}(2023)]%
        {DBLP:conf/kdd/QiBH23}
\bibfield{author}{\bibinfo{person}{Yunzhe Qi}, \bibinfo{person}{Yikun Ban}, {and} \bibinfo{person}{Jingrui He}.} \bibinfo{year}{2023}\natexlab{}.
\newblock \showarticletitle{Graph Neural Bandits}. In \bibinfo{booktitle}{\emph{Proceedings of the 29th {ACM} {SIGKDD} Conference on Knowledge Discovery and Data Mining, {KDD} 2023, Long Beach, CA, USA, August 6-10, 2023}}. \bibinfo{publisher}{{ACM}}, \bibinfo{pages}{1920--1931}.
\newblock


\bibitem[Qiao and Pang(2023)]%
        {DBLP:conf/nips/QiaoP23}
\bibfield{author}{\bibinfo{person}{Hezhe Qiao} {and} \bibinfo{person}{Guansong Pang}.} \bibinfo{year}{2023}\natexlab{}.
\newblock \showarticletitle{Truncated Affinity Maximization: One-class Homophily Modeling for Graph Anomaly Detection}. In \bibinfo{booktitle}{\emph{Advances in Neural Information Processing Systems 36: Annual Conference on Neural Information Processing Systems 2023, NeurIPS 2023, New Orleans, LA, USA, December 10 - 16, 2023}}.
\newblock


\bibitem[Ramakrishnan et~al\mbox{.}(2019)]%
        {DBLP:conf/kdd/RamakrishnanSLS19}
\bibfield{author}{\bibinfo{person}{Jagdish Ramakrishnan}, \bibinfo{person}{Elham Shaabani}, \bibinfo{person}{Chao Li}, {and} \bibinfo{person}{M{\'{a}}ty{\'{a}}s~A. Sustik}.} \bibinfo{year}{2019}\natexlab{}.
\newblock \showarticletitle{Anomaly Detection for an E-commerce Pricing System}. In \bibinfo{booktitle}{\emph{Proceedings of the 25th {ACM} {SIGKDD} International Conference on Knowledge Discovery {\&} Data Mining, {KDD} 2019, Anchorage, AK, USA, August 4-8, 2019}}. \bibinfo{publisher}{{ACM}}, \bibinfo{pages}{1917--1926}.
\newblock


\bibitem[Saunshi et~al\mbox{.}(2019)]%
        {DBLP:conf/icml/SaunshiPAKK19}
\bibfield{author}{\bibinfo{person}{Nikunj Saunshi}, \bibinfo{person}{Orestis Plevrakis}, \bibinfo{person}{Sanjeev Arora}, \bibinfo{person}{Mikhail Khodak}, {and} \bibinfo{person}{Hrishikesh Khandeparkar}.} \bibinfo{year}{2019}\natexlab{}.
\newblock \showarticletitle{A Theoretical Analysis of Contrastive Unsupervised Representation Learning}. In \bibinfo{booktitle}{\emph{Proceedings of the 36th International Conference on Machine Learning, {ICML} 2019, 9-15 June 2019, Long Beach, California, {USA}}}, Vol.~\bibinfo{volume}{97}. \bibinfo{publisher}{{PMLR}}, \bibinfo{pages}{5628--5637}.
\newblock


\bibitem[Shen and Si(2010)]%
        {shen2010non}
\bibfield{author}{\bibinfo{person}{Bin Shen} {and} \bibinfo{person}{Luo Si}.} \bibinfo{year}{2010}\natexlab{}.
\newblock \showarticletitle{Non-Negative Matrix Factorization Clustering on Multiple Manifolds}. In \bibinfo{booktitle}{\emph{Proceedings of the Twenty-Fourth {AAAI} Conference on Artificial Intelligence, {AAAI} 2010, Atlanta, Georgia, USA, July 11-15, 2010}}. \bibinfo{publisher}{{AAAI} Press}.
\newblock


\bibitem[Slipenchuk and Epishkina(2019)]%
        {abdallah2016fraud}
\bibfield{author}{\bibinfo{person}{Pavel Slipenchuk} {and} \bibinfo{person}{Anna Epishkina}.} \bibinfo{year}{2019}\natexlab{}.
\newblock \showarticletitle{Practical User and Entity Behavior Analytics Methods for Fraud Detection Systems in Online Banking: {A} Survey}.
\newblock   \bibinfo{volume}{948} (\bibinfo{year}{2019}), \bibinfo{pages}{83--93}.
\newblock


\bibitem[Sun et~al\mbox{.}(2023)]%
        {sun2023all}
\bibfield{author}{\bibinfo{person}{Xiangguo Sun}, \bibinfo{person}{Hong Cheng}, \bibinfo{person}{Jia Li}, \bibinfo{person}{Bo Liu}, {and} \bibinfo{person}{Jihong Guan}.} \bibinfo{year}{2023}\natexlab{}.
\newblock \showarticletitle{All in one: Multi-task prompting for graph neural networks}. In \bibinfo{booktitle}{\emph{Proceedings of the 29th ACM SIGKDD Conference on Knowledge Discovery and Data Mining}}. \bibinfo{pages}{2120--2131}.
\newblock


\bibitem[Sun et~al\mbox{.}(2021a)]%
        {sun2021heterogeneous}
\bibfield{author}{\bibinfo{person}{Xiangguo Sun}, \bibinfo{person}{Hongzhi Yin}, \bibinfo{person}{Bo Liu}, \bibinfo{person}{Hongxu Chen}, \bibinfo{person}{Jiuxin Cao}, \bibinfo{person}{Yingxia Shao}, {and} \bibinfo{person}{Nguyen~Quoc Viet~Hung}.} \bibinfo{year}{2021}\natexlab{a}.
\newblock \showarticletitle{Heterogeneous hypergraph embedding for graph classification}. In \bibinfo{booktitle}{\emph{Proceedings of the 14th ACM international conference on web search and data mining}}. \bibinfo{pages}{725--733}.
\newblock


\bibitem[Sun et~al\mbox{.}(2021b)]%
        {sun2021multi}
\bibfield{author}{\bibinfo{person}{Xiangguo Sun}, \bibinfo{person}{Hongzhi Yin}, \bibinfo{person}{Bo Liu}, \bibinfo{person}{Hongxu Chen}, \bibinfo{person}{Qing Meng}, \bibinfo{person}{Wang Han}, {and} \bibinfo{person}{Jiuxin Cao}.} \bibinfo{year}{2021}\natexlab{b}.
\newblock \showarticletitle{Multi-level hyperedge distillation for social linking prediction on sparsely observed networks}. In \bibinfo{booktitle}{\emph{Proceedings of the Web Conference 2021}}. \bibinfo{pages}{2934--2945}.
\newblock


\bibitem[Tang and Liu(2009)]%
        {DBLP:conf/kdd/TangL09}
\bibfield{author}{\bibinfo{person}{Lei Tang} {and} \bibinfo{person}{Huan Liu}.} \bibinfo{year}{2009}\natexlab{}.
\newblock \showarticletitle{Relational learning via latent social dimensions}. In \bibinfo{booktitle}{\emph{Proceedings of the 15th {ACM} {SIGKDD} International Conference on Knowledge Discovery and Data Mining, Paris, France, June 28 - July 1, 2009}}. \bibinfo{publisher}{{ACM}}, \bibinfo{pages}{817--826}.
\newblock


\bibitem[Tong and Lin(2011)]%
        {tong2011non}
\bibfield{author}{\bibinfo{person}{Hanghang Tong} {and} \bibinfo{person}{Ching{-}Yung Lin}.} \bibinfo{year}{2011}\natexlab{}.
\newblock \showarticletitle{Non-Negative Residual Matrix Factorization with Application to Graph Anomaly Detection}. In \bibinfo{booktitle}{\emph{Proceedings of the Eleventh {SIAM} International Conference on Data Mining, {SDM} 2011, April 28-30, 2011, Mesa, Arizona, {USA}}}. \bibinfo{publisher}{{SIAM} / Omnipress}, \bibinfo{pages}{143--153}.
\newblock


\bibitem[Tsai et~al\mbox{.}(2022)]%
        {DBLP:conf/iclr/TsaiLLLSM22}
\bibfield{author}{\bibinfo{person}{Yao{-}Hung~Hubert Tsai}, \bibinfo{person}{Tianqin Li}, \bibinfo{person}{Weixin Liu}, \bibinfo{person}{Peiyuan Liao}, \bibinfo{person}{Ruslan Salakhutdinov}, {and} \bibinfo{person}{Louis{-}Philippe Morency}.} \bibinfo{year}{2022}\natexlab{}.
\newblock \showarticletitle{Learning Weakly-supervised Contrastive Representations}. In \bibinfo{booktitle}{\emph{The Tenth International Conference on Learning Representations, {ICLR} 2022, Virtual Event, April 25-29, 2022}}. \bibinfo{publisher}{OpenReview.net}.
\newblock


\bibitem[van~den Oord et~al\mbox{.}(2018)]%
        {oord2018representation}
\bibfield{author}{\bibinfo{person}{A{\"{a}}ron van~den Oord}, \bibinfo{person}{Yazhe Li}, {and} \bibinfo{person}{Oriol Vinyals}.} \bibinfo{year}{2018}\natexlab{}.
\newblock \showarticletitle{Representation Learning with Contrastive Predictive Coding}.
\newblock \bibinfo{journal}{\emph{CoRR}}  \bibinfo{volume}{abs/1807.03748} (\bibinfo{year}{2018}).
\newblock


\bibitem[von Luxburg(2007)]%
        {DBLP:journals/sac/Luxburg07}
\bibfield{author}{\bibinfo{person}{Ulrike von Luxburg}.} \bibinfo{year}{2007}\natexlab{}.
\newblock \showarticletitle{A tutorial on spectral clustering}.
\newblock \bibinfo{journal}{\emph{Stat. Comput.}} \bibinfo{volume}{17}, \bibinfo{number}{4} (\bibinfo{year}{2007}), \bibinfo{pages}{395--416}.
\newblock
\href{https://doi.org/10.1007/S11222-007-9033-Z}{doi:\nolinkurl{10.1007/S11222-007-9033-Z}}


\bibitem[Wang et~al\mbox{.}(2019b)]%
        {wang2019semi}
\bibfield{author}{\bibinfo{person}{Daixin Wang}, \bibinfo{person}{Jianbin Lin}, \bibinfo{person}{Peng Cui}, \bibinfo{person}{Quanhui Jia}, \bibinfo{person}{Zhen Wang}, \bibinfo{person}{Yanming Fang}, \bibinfo{person}{Quan Yu}, \bibinfo{person}{Jun Zhou}, \bibinfo{person}{Shuang Yang}, {and} \bibinfo{person}{Yuan Qi}.} \bibinfo{year}{2019}\natexlab{b}.
\newblock \showarticletitle{A semi-supervised graph attentive network for financial fraud detection}. In \bibinfo{booktitle}{\emph{2019 IEEE International Conference on Data Mining (ICDM)}}. IEEE, \bibinfo{pages}{598--607}.
\newblock


\bibitem[Wang and Isola(2020)]%
        {DBLP:conf/icml/0001I20}
\bibfield{author}{\bibinfo{person}{Tongzhou Wang} {and} \bibinfo{person}{Phillip Isola}.} \bibinfo{year}{2020}\natexlab{}.
\newblock \showarticletitle{Understanding Contrastive Representation Learning through Alignment and Uniformity on the Hypersphere}. In \bibinfo{booktitle}{\emph{Proceedings of the 37th International Conference on Machine Learning, {ICML} 2020, 13-18 July 2020, Virtual Event}}, Vol.~\bibinfo{volume}{119}. \bibinfo{publisher}{{PMLR}}, \bibinfo{pages}{9929--9939}.
\newblock


\bibitem[Wang et~al\mbox{.}(2019a)]%
        {DBLP:conf/www/WangJSWYCY19}
\bibfield{author}{\bibinfo{person}{Xiao Wang}, \bibinfo{person}{Houye Ji}, \bibinfo{person}{Chuan Shi}, \bibinfo{person}{Bai Wang}, \bibinfo{person}{Yanfang Ye}, \bibinfo{person}{Peng Cui}, {and} \bibinfo{person}{Philip~S. Yu}.} \bibinfo{year}{2019}\natexlab{a}.
\newblock \showarticletitle{Heterogeneous Graph Attention Network}. In \bibinfo{booktitle}{\emph{The World Wide Web Conference, {WWW} 2019, San Francisco, CA, USA, May 13-17, 2019}}. \bibinfo{publisher}{{ACM}}, \bibinfo{pages}{2022--2032}.
\newblock


\bibitem[Wang et~al\mbox{.}(2023)]%
        {DBLP:journals/tkdd/WangCLFZZ23}
\bibfield{author}{\bibinfo{person}{Yu Wang}, \bibinfo{person}{Chuan Chen}, \bibinfo{person}{Jinrong Lai}, \bibinfo{person}{Lele Fu}, \bibinfo{person}{Yuren Zhou}, {and} \bibinfo{person}{Zibin Zheng}.} \bibinfo{year}{2023}\natexlab{}.
\newblock \showarticletitle{A Self-Representation Method with Local Similarity Preserving for Fast Multi-View Outlier Detection}.
\newblock \bibinfo{journal}{\emph{{ACM} Trans. Knowl. Discov. Data}} \bibinfo{volume}{17}, \bibinfo{number}{1} (\bibinfo{year}{2023}), \bibinfo{pages}{2:1--2:20}.
\newblock


\bibitem[Xu et~al\mbox{.}(2022)]%
        {DBLP:conf/pakdd/XuHZDL22}
\bibfield{author}{\bibinfo{person}{Zhiming Xu}, \bibinfo{person}{Xiao Huang}, \bibinfo{person}{Yue Zhao}, \bibinfo{person}{Yushun Dong}, {and} \bibinfo{person}{Jundong Li}.} \bibinfo{year}{2022}\natexlab{}.
\newblock \showarticletitle{Contrastive Attributed Network Anomaly Detection with Data Augmentation}. In \bibinfo{booktitle}{\emph{Advances in Knowledge Discovery and Data Mining - 26th Pacific-Asia Conference, {PAKDD} 2022, Chengdu, China, May 16-19, 2022, Proceedings, Part {II}}} \emph{(\bibinfo{series}{Lecture Notes in Computer Science}, Vol.~\bibinfo{volume}{13281})}. \bibinfo{publisher}{Springer}, \bibinfo{pages}{444--457}.
\newblock


\bibitem[Ying et~al\mbox{.}(2018)]%
        {DBLP:conf/nips/YingY0RHL18}
\bibfield{author}{\bibinfo{person}{Zhitao Ying}, \bibinfo{person}{Jiaxuan You}, \bibinfo{person}{Christopher Morris}, \bibinfo{person}{Xiang Ren}, \bibinfo{person}{William~L. Hamilton}, {and} \bibinfo{person}{Jure Leskovec}.} \bibinfo{year}{2018}\natexlab{}.
\newblock \showarticletitle{Hierarchical Graph Representation Learning with Differentiable Pooling}. In \bibinfo{booktitle}{\emph{Advances in Neural Information Processing Systems 31: Annual Conference on Neural Information Processing Systems 2018, NeurIPS 2018, December 3-8, 2018, Montr{\'{e}}al, Canada}}. \bibinfo{pages}{4805--4815}.
\newblock


\bibitem[Zhang et~al\mbox{.}(2015)]%
        {DBLP:conf/aaai/ZhangZLY15}
\bibfield{author}{\bibinfo{person}{Xianchao Zhang}, \bibinfo{person}{Linlin Zong}, \bibinfo{person}{Xinyue Liu}, {and} \bibinfo{person}{Hong Yu}.} \bibinfo{year}{2015}\natexlab{}.
\newblock \showarticletitle{Constrained NMF-Based Multi-View Clustering on Unmapped Data}. In \bibinfo{booktitle}{\emph{Proceedings of the Twenty-Ninth {AAAI} Conference on Artificial Intelligence, January 25-30, 2015, Austin, Texas, {USA}}}. \bibinfo{publisher}{{AAAI} Press}, \bibinfo{pages}{3174--3180}.
\newblock


\bibitem[Zheng et~al\mbox{.}(2024a)]%
        {zheng2024online}
\bibfield{author}{\bibinfo{person}{Lecheng Zheng}, \bibinfo{person}{Zhengzhang Chen}, \bibinfo{person}{Haifeng Chen}, {and} \bibinfo{person}{Jingrui He}.} \bibinfo{year}{2024}\natexlab{a}.
\newblock \showarticletitle{Online Multi-modal Root Cause Analysis}.
\newblock \bibinfo{journal}{\emph{arXiv preprint arXiv:2410.10021}} (\bibinfo{year}{2024}).
\newblock


\bibitem[Zheng et~al\mbox{.}(2024b)]%
        {DBLP:conf/www/ZhengCHC24}
\bibfield{author}{\bibinfo{person}{Lecheng Zheng}, \bibinfo{person}{Zhengzhang Chen}, \bibinfo{person}{Jingrui He}, {and} \bibinfo{person}{Haifeng Chen}.} \bibinfo{year}{2024}\natexlab{b}.
\newblock \showarticletitle{{MULAN:} Multi-modal Causal Structure Learning and Root Cause Analysis for Microservice Systems}. In \bibinfo{booktitle}{\emph{Proceedings of the {ACM} on Web Conference 2024, {WWW} 2024, Singapore, May 13-17, 2024}}. \bibinfo{publisher}{{ACM}}, \bibinfo{pages}{4107--4116}.
\newblock


\bibitem[Zheng et~al\mbox{.}(2024c)]%
        {zheng2024drgnn}
\bibfield{author}{\bibinfo{person}{Lecheng Zheng}, \bibinfo{person}{Dongqi Fu}, \bibinfo{person}{Ross Maciejewski}, {and} \bibinfo{person}{Jingrui He}.} \bibinfo{year}{2024}\natexlab{c}.
\newblock \showarticletitle{Drgnn: Deep residual graph neural network with contrastive learning}.
\newblock \bibinfo{journal}{\emph{Transactions on Machine Learning Research}} (\bibinfo{year}{2024}).
\newblock


\bibitem[Zheng et~al\mbox{.}(2024d)]%
        {DBLP:conf/kdd/ZhengJLTH24}
\bibfield{author}{\bibinfo{person}{Lecheng Zheng}, \bibinfo{person}{Baoyu Jing}, \bibinfo{person}{Zihao Li}, \bibinfo{person}{Hanghang Tong}, {and} \bibinfo{person}{Jingrui He}.} \bibinfo{year}{2024}\natexlab{d}.
\newblock \showarticletitle{Heterogeneous Contrastive Learning for Foundation Models and Beyond}. In \bibinfo{booktitle}{\emph{Proceedings of the 30th {ACM} {SIGKDD} Conference on Knowledge Discovery and Data Mining, {KDD} 2024, Barcelona, Spain, August 25-29, 2024}}. \bibinfo{publisher}{{ACM}}, \bibinfo{pages}{6666--6676}.
\newblock


\bibitem[Zheng et~al\mbox{.}(2024e)]%
        {zheng2024pyg}
\bibfield{author}{\bibinfo{person}{Lecheng Zheng}, \bibinfo{person}{Baoyu Jing}, \bibinfo{person}{Zihao Li}, \bibinfo{person}{Zhichen Zeng}, \bibinfo{person}{Tianxin Wei}, \bibinfo{person}{Mengting Ai}, \bibinfo{person}{Xinrui He}, \bibinfo{person}{Lihui Liu}, \bibinfo{person}{Dongqi Fu}, \bibinfo{person}{Jiaxuan You}, {et~al\mbox{.}}} \bibinfo{year}{2024}\natexlab{e}.
\newblock \showarticletitle{PyG-SSL: A Graph Self-Supervised Learning Toolkit}.
\newblock \bibinfo{journal}{\emph{arXiv preprint arXiv:2412.21151}} (\bibinfo{year}{2024}).
\newblock


\bibitem[Zheng et~al\mbox{.}(2022)]%
        {DBLP:conf/kdd/ZhengXZH22}
\bibfield{author}{\bibinfo{person}{Lecheng Zheng}, \bibinfo{person}{Jinjun Xiong}, \bibinfo{person}{Yada Zhu}, {and} \bibinfo{person}{Jingrui He}.} \bibinfo{year}{2022}\natexlab{}.
\newblock \showarticletitle{Contrastive Learning with Complex Heterogeneity}. In \bibinfo{booktitle}{\emph{{KDD} '22: The 28th {ACM} {SIGKDD} Conference on Knowledge Discovery and Data Mining, Washington, DC, USA, August 14 - 18, 2022}}. \bibinfo{publisher}{{ACM}}, \bibinfo{pages}{2594--2604}.
\newblock


\bibitem[Zheng et~al\mbox{.}(2024f)]%
        {DBLP:conf/icde/Zheng0TXZH24}
\bibfield{author}{\bibinfo{person}{Lecheng Zheng}, \bibinfo{person}{Dawei Zhou}, \bibinfo{person}{Hanghang Tong}, \bibinfo{person}{Jiejun Xu}, \bibinfo{person}{Yada Zhu}, {and} \bibinfo{person}{Jingrui He}.} \bibinfo{year}{2024}\natexlab{f}.
\newblock \showarticletitle{Fairgen: Towards Fair Graph Generation}. In \bibinfo{booktitle}{\emph{40th {IEEE} International Conference on Data Engineering, {ICDE} 2024, Utrecht, The Netherlands, May 13-16, 2024}}. \bibinfo{publisher}{{IEEE}}, \bibinfo{pages}{2285--2297}.
\newblock


\bibitem[Zheng et~al\mbox{.}(2023)]%
        {DBLP:conf/sdm/ZhengZH23}
\bibfield{author}{\bibinfo{person}{Lecheng Zheng}, \bibinfo{person}{Yada Zhu}, {and} \bibinfo{person}{Jingrui He}.} \bibinfo{year}{2023}\natexlab{}.
\newblock \showarticletitle{Fairness-aware Multi-view Clustering}. In \bibinfo{booktitle}{\emph{Proceedings of the 2023 {SIAM} International Conference on Data Mining, {SDM} 2023, Minneapolis-St. Paul Twin Cities, MN, USA, April 27-29, 2023}}. \bibinfo{publisher}{{SIAM}}, \bibinfo{pages}{856--864}.
\newblock


\bibitem[Zheng et~al\mbox{.}(2021b)]%
        {DBLP:conf/iccv/Zheng0Y0Z0021}
\bibfield{author}{\bibinfo{person}{Mingkai Zheng}, \bibinfo{person}{Fei Wang}, \bibinfo{person}{Shan You}, \bibinfo{person}{Chen Qian}, \bibinfo{person}{Changshui Zhang}, \bibinfo{person}{Xiaogang Wang}, {and} \bibinfo{person}{Chang Xu}.} \bibinfo{year}{2021}\natexlab{b}.
\newblock \showarticletitle{Weakly Supervised Contrastive Learning}. In \bibinfo{booktitle}{\emph{2021 {IEEE/CVF} International Conference on Computer Vision, {ICCV} 2021, Montreal, QC, Canada, October 10-17, 2021}}. \bibinfo{publisher}{{IEEE}}, \bibinfo{pages}{10022--10031}.
\newblock


\bibitem[Zheng et~al\mbox{.}(2021a)]%
        {zheng2021generative}
\bibfield{author}{\bibinfo{person}{Yu Zheng}, \bibinfo{person}{Ming Jin}, \bibinfo{person}{Yixin Liu}, \bibinfo{person}{Lianhua Chi}, \bibinfo{person}{Khoa~T Phan}, {and} \bibinfo{person}{Yi-Ping~Phoebe Chen}.} \bibinfo{year}{2021}\natexlab{a}.
\newblock \showarticletitle{Generative and contrastive self-supervised learning for graph anomaly detection}.
\newblock \bibinfo{journal}{\emph{IEEE Transactions on Knowledge and Data Engineering}} (\bibinfo{year}{2021}).
\newblock


\bibitem[Zhong et~al\mbox{.}(2020)]%
        {zhong2020financial}
\bibfield{author}{\bibinfo{person}{Qiwei Zhong}, \bibinfo{person}{Yang Liu}, \bibinfo{person}{Xiang Ao}, \bibinfo{person}{Binbin Hu}, \bibinfo{person}{Jinghua Feng}, \bibinfo{person}{Jiayu Tang}, {and} \bibinfo{person}{Qing He}.} \bibinfo{year}{2020}\natexlab{}.
\newblock \showarticletitle{Financial Defaulter Detection on Online Credit Payment via Multi-view Attributed Heterogeneous Information Network}. In \bibinfo{booktitle}{\emph{{WWW} '20: The Web Conference 2020, Taipei, Taiwan, April 20-24, 2020}}. \bibinfo{publisher}{{ACM} / {IW3C2}}, \bibinfo{pages}{785--795}.
\newblock


\bibitem[Zhou et~al\mbox{.}(2022)]%
        {DBLP:conf/cikm/ZhouZF0H22}
\bibfield{author}{\bibinfo{person}{Dawei Zhou}, \bibinfo{person}{Lecheng Zheng}, \bibinfo{person}{Dongqi Fu}, \bibinfo{person}{Jiawei Han}, {and} \bibinfo{person}{Jingrui He}.} \bibinfo{year}{2022}\natexlab{}.
\newblock \showarticletitle{MentorGNN: Deriving Curriculum for Pre-Training GNNs}. In \bibinfo{booktitle}{\emph{Proceedings of the 31st {ACM} International Conference on Information {\&} Knowledge Management, Atlanta, GA, USA, October 17-21, 2022}}. \bibinfo{publisher}{{ACM}}, \bibinfo{pages}{2721--2731}.
\newblock


\bibitem[Zhou et~al\mbox{.}(2020)]%
        {DBLP:conf/kdd/ZhouZ0H20}
\bibfield{author}{\bibinfo{person}{Dawei Zhou}, \bibinfo{person}{Lecheng Zheng}, \bibinfo{person}{Jiawei Han}, {and} \bibinfo{person}{Jingrui He}.} \bibinfo{year}{2020}\natexlab{}.
\newblock \showarticletitle{A Data-Driven Graph Generative Model for Temporal Interaction Networks}. In \bibinfo{booktitle}{\emph{{KDD} '20: The 26th {ACM} {SIGKDD} Conference on Knowledge Discovery and Data Mining, Virtual Event, CA, USA, August 23-27, 2020}}. \bibinfo{publisher}{{ACM}}, \bibinfo{pages}{401--411}.
\newblock


\bibitem[Zhou et~al\mbox{.}(2019)]%
        {DBLP:journals/fdata/ZhouZXH19}
\bibfield{author}{\bibinfo{person}{Dawei Zhou}, \bibinfo{person}{Lecheng Zheng}, \bibinfo{person}{Jiejun Xu}, {and} \bibinfo{person}{Jingrui He}.} \bibinfo{year}{2019}\natexlab{}.
\newblock \showarticletitle{Misc-GAN: {A} Multi-scale Generative Model for Graphs}.
\newblock \bibinfo{journal}{\emph{Frontiers Big Data}}  \bibinfo{volume}{2} (\bibinfo{year}{2019}), \bibinfo{pages}{3}.
\newblock


\bibitem[Zhou et~al\mbox{.}(2021)]%
        {DBLP:journals/tii/ZhouHLMJ21}
\bibfield{author}{\bibinfo{person}{Xiaokang Zhou}, \bibinfo{person}{Yiyong Hu}, \bibinfo{person}{Wei Liang}, \bibinfo{person}{Jianhua Ma}, {and} \bibinfo{person}{Qun Jin}.} \bibinfo{year}{2021}\natexlab{}.
\newblock \showarticletitle{Variational {LSTM} Enhanced Anomaly Detection for Industrial Big Data}.
\newblock \bibinfo{journal}{\emph{{IEEE} Trans. Ind. Informatics}} \bibinfo{volume}{17}, \bibinfo{number}{5} (\bibinfo{year}{2021}), \bibinfo{pages}{3469--3477}.
\newblock


\bibitem[Zhu et~al\mbox{.}(2021)]%
        {DBLP:conf/nips/ZhuSK21}
\bibfield{author}{\bibinfo{person}{Hao Zhu}, \bibinfo{person}{Ke Sun}, {and} \bibinfo{person}{Peter Koniusz}.} \bibinfo{year}{2021}\natexlab{}.
\newblock \showarticletitle{Contrastive Laplacian Eigenmaps}. In \bibinfo{booktitle}{\emph{Advances in Neural Information Processing Systems 34: Annual Conference on Neural Information Processing Systems 2021, NeurIPS 2021, December 6-14, 2021, virtual}}. \bibinfo{pages}{5682--5695}.
\newblock


\end{thebibliography}

\appendix
\appendix

\section{Proof}
\subsection{Proof for Lemma \ref{lemma_1}}
\label{appendix_2}
\textbf{Lemma \ref{lemma_1}:} \textit{
(Similarity-guided Graph Contrastive Loss) Let $\bm{\bar{M}}$ be the output of a one-layer graph neural network defined in Eq.~\ref{graph_pooling}. Then, $\mathcal{L}_C$ is equivalent to the following loss function:
\begin{center}
    $\mathcal{L}_C = \mathcal{L}_f + C$
\end{center}
where $\mathcal{L}_f = -\sum_{i=1}^n\sum_{j=1}^n \log\frac{\exp(2\bm{\tilde{A}}_{ij}\bm{\bar{h}}_i\bm{\bar{h}}_j^T)}{\Pi_{k=1}^n \exp((\bm{\bar{h}}_i\bm{\bar{h}}_k^T)^2)^{1/n}}$ is a graph contrastive loss and $C$ is a constant.}

\begin{proof}
\begin{equation}
\begin{split}
    \nonumber \min_{\bar{H}} \mathcal{L}_C &=\min_{\bar{H}} ||\bm{\tilde{A}}-\bm{\bar{H}}\bm{\bar{H}}^T||_F^2 \\
    &=\min_{\bar{H}} \sum_{i=1}^n\sum_{j=1}^n (\bm{\tilde{A}}_{ij}-\bm{\bar{h}}_i\bm{\bar{h}}_j^T)^2 \\
    &=\min_{\bar{H}} \sum_{i=1}^n\sum_{j=1}^n (\bm{\tilde{A}}_{ij}^2-2\bm{\tilde{A}}_{ij}\bm{\bar{h}}_i\bm{\bar{h}}_j^T + (\bm{\bar{h}}_i\bm{\bar{h}}_j^T)^2)
\end{split}
\end{equation}
Notice that $\bm{\bar{M}}$ is independent of $\bm{\bar{H}}$. When we fix the parameter $\bm{\bar{M}}$ to update $\bm{\bar{H}}$, then $\bm{\tilde{A}}_{ij}^2$ can be considered as a constant in this optimization problem. Thus, we have 
\begin{equation}
\begin{split}
    \nonumber
    \min_{\bar{H}} \mathcal{L}_C &=\min_{\bar{H}} \sum_{i=1}^n\sum_{j=1}^n (-2\bm{\tilde{A}}_{ij}\bm{\bar{h}}_i\bm{\bar{h}}_j^T + (\bm{\bar{h}}_i\bm{\bar{h}}_j^T)^2) + C \\
    &=\min_{\bar{H}} \sum_{i=1}^n \sum_{j=1}^n(-2\bm{\tilde{A}}_{ij}\bm{\bar{h}}_i\bm{\bar{h}}_j^T +  \frac{1}{n}\sum_{k=1}^n(\bm{\bar{h}}_i\bm{\bar{h}}_k^T)^2) + C \\
    &=\min_{\bar{H}} -\sum_{i=1}^n \sum_{j=1}^n\log \frac{\exp(2\bm{\tilde{A}}_{ij}\bm{\bar{h}}_i\bm{\bar{h}}_j^T)}{\Pi_{k=1}^n\exp((\bm{\bar{h}}_i\bm{\bar{h}}_k^T)^2)^{\frac{1}{n}}} + C \\
    &= \min_{\bar{H}} \mathcal{L}_f + C \\
\end{split}
\end{equation}
where $\mathcal{L}_f = -\sum_{i=1}^n\sum_{j=1}^n \log\frac{\exp(2\bm{\tilde{A}}_{ij}\bm{\bar{h}}_i\bm{\bar{h}}_j^T)}{\Pi_{k=1}^n \exp((\bm{\bar{h}}_i\bm{\bar{h}}_k^T)^2)^{1/n}}$. Thus, we have $\mathcal{L}_C=\mathcal{L}_f + C$, which completes the proof. \\
\end{proof}

\subsection{Proof for Lemma \ref{lemma_2}}
\label{appendix_3}
\textbf{Lemma \ref{lemma_2}:}\textit{ (Graph Contrastive Spectral Clustering) Let $\bm{\bar{M}}$ be the output of a one-layer graph neural network defined in Eq.~\ref{graph_pooling} and $\bm{\bar{h}}_i$ and $\bm{\bar{h}}_j$ be unit vectors. Then, minimizing $\mathcal{L}_C$ is equivalent to minimizing the following loss function:
\begin{center}
    $\min \mathcal{L}_C = \min [2Tr(\bm{\bar{H}}^T \bm{L} \bm{\bar{H}}) + R(\bm{\bar{H}})]$
\end{center}
where $\bm{L}=\bm{I}-\bm{\tilde{A}}$ can be considered as the normalized graph Laplacian, $\bm{I}$ is the identity matrix and $R(\bm{\bar{H}})=||\bm{\bar{H}}\bm{\bar{H}}^T||_F^2$ is the regularization term.}

\begin{proof}
Based on the proof in Lemma~\ref{lemma_1}, we have
\begin{equation*}
\begin{split}
    \nonumber
    \min\mathcal{L}_C &= \sum_{i=1}^n\sum_{j=1}^n (-2\bm{\tilde{A}}_{ij}\bm{\bar{h}}_i\bm{\bar{h}}_j^T + (\bm{\bar{h}}_i\bm{\bar{h}}_j^T)^2)\\
    &=(\sum_{i=1}^n\sum_{j=1}^n (2\bm{\tilde{A}}_{ij}-2\bm{\tilde{A}}_{ij}-2\bm{\tilde{A}}_{ij}\bm{\bar{h}}_i\bm{\bar{h}}_j^T +  (\bm{\bar{h}}_i\bm{\bar{h}}_j^T)^2)\\
\end{split}
\end{equation*}
Notice that $\bm{\tilde{A}}=\tilde{D}^{-1/2}(\alpha \bm{\bar{M}}\bm{\bar{M}}^T + (1-\alpha)\frac{1}{v}\sum_{a=1}^v\bm{A}^a)\tilde{D}^{-1/2}$ and $\tilde{D}_{ii} =\sum_j (\alpha \bm{\bar{M}}_i\bm{\bar{M}}_j^T + (1-\alpha)\frac{1}{v}\sum_{a=1}^v\bm{A}_{ij}^a)$. We have $\sum_{j=1}^n \bm{\tilde{A}}_{ij}=1$ and thus $\sum_{i=1}^n\sum_{j=1}^n2\bm{\tilde{A}}_{ij}$ is a constant, which can be ignored in this optimization problem.
Since $\bm{\bar{h}}_i$ and $\bm{\bar{h}}_j$ are unit vectors, we have
\begin{equation}
\label{lemma2_proof}
\begin{split}
    \min\mathcal{L}_C &=\sum_{i=1}^n\sum_{j=1}^n (2\bm{\tilde{A}}_{ij}-2\bm{\tilde{A}}_{ij}\bm{\bar{h}}_i\bm{\bar{h}}_j^T + (\bm{\bar{h}}_i\bm{\bar{h}}_j^T)^2)\\
    &=\sum_{i=1}^n\sum_{j=1}^n \bm{\tilde{A}}_{ij}||\bm{\bar{h}}_i||_2^2 + \sum_{i=1}^n\sum_{j=1}^n\bm{\tilde{A}}_{ij}||\bm{\bar{h}}_j||_2^2 -\sum_{i=1}^n\sum_{j=1}^n2\bm{\tilde{A}}_{ij}\bm{\bar{h}}_i\bm{\bar{h}}_j^T \\
    & +\sum_{i=1}^n\sum_{j=1}^n (\bm{\bar{h}}_i\bm{\bar{h}}_j^T)^2 \\
    &=  \sum_{i=1}^n\bm{I}_{ii}||\bm{\bar{h}}_i||_2^2 + \sum_{j=1}^n\bm{I}_{jj}||\bm{\bar{h}}_j||_2^2 -\sum_{i=1}^n\sum_{j=1}^n2\bm{\tilde{A}}_{ij}\bm{\bar{h}}_i\bm{\bar{h}}_j^T\\
    &  + \sum_{i=1}^n\sum_{j=1}^n (\bm{\bar{h}}_i\bm{\bar{h}}_j^T)^2 \\
    &=  2Tr(\bm{\bar{H}}^T \bm{L}\bm{\bar{H}}) + R(\bm{\bar{H}})
\end{split}
\end{equation}
where $\bm{L}=\bm{I}-\bm{\tilde{A}}$ can be considered as the normalized graph Laplacian, $\bm{I}$ is the identity matrix and $R(\bm{\bar{H}})=||\bm{\bar{H}}\bm{\bar{H}}^T||_F^2$, which completes the proof. 
\end{proof}

\subsection{Proof for theorem~\ref{theorem_1}}
\label{appendix_1}

\textbf{Definition \ref{definition_1}} \textit{Given a sample $\bm{x_i}$, we say ($\bm{x}_i$, $\bm{x}_j$) is a false negative pair (or a true positive pair), if their optimal representations satisfy $\exp(\bm{\bar{h}}_i\bm{\bar{h}}_j^T/\tau) > 1$ for a small positive value $\tau$. Similarly, we say ($\bm{x}_i$, $\bm{x}_k$) is a true negative pair (or a false positive pair), if their optimal representations satisfy $\exp(\bm{\bar{h}}_i\bm{\bar{h}}_k^T/\tau) \approx 0$ for a small positive value $\tau$. }

\textbf{Theorem \ref{theorem_1}}
\textit{
Given the contrastive learning loss function $\mathcal{L}_3$, if there exists one false positive sample in the batch during training, the contrastive learning loss will lead to a sub-optimal solution.}

\begin{proof}
We can rewrite $\mathcal{L}_3$ as follows:
\begin{align}
    \nonumber \mathcal{L}_3 
        \nonumber &= \sum_{i}\sum_{j\in C(i), j\neq i}[\log (\frac{\exp(\bm{\bar{h}}_i\bm{\bar{h}}_j^T/\tau) + \sum_{k \notin C(i)} \exp(\bm{\bar{h}}_i\bm{\bar{h}}_k^T/\tau)}{\exp(\bm{\bar{h}}_i\bm{\bar{h}}_j^T/\tau)})] \\
        \nonumber &= \sum_{i}\sum_{j\in C(i), j\neq i}[\log (\exp(\bm{\bar{h}}_i\bm{\bar{h}}_j^T/\tau) + \sum_{k\notin C(i)} \exp(\bm{\bar{h}}_i\bm{\bar{h}}_k^T/\tau))- \bm{\bar{h}}_i\bm{\bar{h}}_j^T/\tau] \\
        \nonumber &= \sum_{j\in C(1), j\neq 1}[\log (\exp(\bm{\bar{h}}_1\bm{\bar{h}}_j^T/\tau) + \sum_{k\notin C(1)} \exp(\bm{\bar{h}}_1\bm{\bar{h}}_k^T/\tau))- \bm{\bar{h}}_1\bm{\bar{h}}_j^T/\tau] \\
        \nonumber &+ \sum_{i\neq 1}\sum_{j\in C(i), j\neq i}[\log (\exp(\bm{\bar{h}}_i\bm{\bar{h}}_j^T/\tau) + \sum_{k\notin C(i)} \exp(\bm{\bar{h}}_i\bm{\bar{h}}_k^T/\tau))- \bm{\bar{h}}_i\bm{\bar{h}}_j^T/\tau] \\
\end{align}
Here, we select $\bm{x}_1$ as the anchor node such that ($\bm{x}_i$,$\bm{x}_1$) is a true positive pair (\ie, $\bm{x}_i$ and $\bm{x}_1$ are from the same cluster). Taking the derivative of $\mathcal{L}_3$ with respect to $\bm{\bar{h}}_1$, we have

\begin{align}
    \nonumber \frac{\partial \mathcal{L}_3}{\partial \bm{\bar{h}}_1} &= \frac{1}{\tau} \sum_{j\in C(1), j\neq 1} [\frac{\exp(\bm{\bar{h}}_1\bm{\bar{h}}_j^T/\tau)\bm{\bar{h}}_j + \sum_{k\notin C(1)} \exp(\bm{\bar{h}}_1\bm{\bar{h}}_k^T/\tau) \bm{\bar{h}}_k}{\exp(\bm{\bar{h}}_1\bm{\bar{h}}_j^T/\tau) + \sum_{k\notin C(1)} \exp(\bm{\bar{h}}_1\bm{\bar{h}}_k^T/\tau)}  - \bm{\bar{h}}_j] \\
    &\nonumber + \frac{1}{\tau} \sum_{i \neq 1, i\in C(1)}  [\frac{\exp(\bm{\bar{h}}_i\bm{\bar{h}}_1^T/\tau)\bm{\bar{h}}_i}{\exp(\bm{\bar{h}}_i\bm{\bar{h}}_1^T/\tau) + \sum_{k\notin C(i)} \exp(\bm{\bar{h}}_i\bm{\bar{h}}_k^T/\tau)} - \bm{\bar{h}}_i] \\
    & \nonumber = \frac{1}{\tau} \sum_{j\in C(1), j\neq 1} [\frac{\sum_{k\notin C(1)} (\exp(\bm{\bar{h}}_1\bm{\bar{h}}_k^T/\tau) \bm{\bar{h}}_k -\exp(\bm{\bar{h}}_1\bm{\bar{h}}_k^T/\tau)\bm{\bar{h}}_j)}{\exp(\bm{\bar{h}}_1\bm{\bar{h}}_j^T/\tau) + \sum_{k\notin C(1)} \exp(\bm{\bar{h}}_1\bm{\bar{h}}_k^T/\tau)}] \\
    &+ \nonumber \frac{1}{\tau} \sum_{i \neq 1, i\in C(1)} [\frac{-\sum_{k \notin C(i)} \exp(\bm{\bar{h}}_i\bm{\bar{h}}_k^T/\tau)\bm{\bar{h}}_i}{\exp(\bm{\bar{h}}_i\bm{\bar{h}}_1^T/\tau) + \sum_{k\notin C(i)} \exp(\bm{\bar{h}}_i\bm{\bar{h}}_k^T/\tau)}]
\end{align}
Setting the gradient to 0, we have 
\begin{align}
     &\nonumber \sum_{j\in C(1), j\neq 1} [\frac{\sum_{k \notin C(1)} [\exp(\bm{\bar{h}}_1\bm{\bar{h}}_k^T/\tau)\bm{\bar{h}}_j - \exp(\bm{\bar{h}}_1\bm{\bar{h}}_k^T/\tau) \bm{\bar{h}}_k]}{\exp(\bm{\bar{h}}_1\bm{\bar{h}}_j^T/\tau) + \sum_{k\notin C(1)} \exp(\bm{\bar{h}}_1\bm{\bar{h}}_k^T/\tau)}] \\
     &+ \sum_{i \neq 1, i\in C(1)} [\frac{\sum_{k \notin C(i)} \exp(\bm{\bar{h}}_i\bm{\bar{h}}_k^T/\tau)\bm{\bar{h}}_i}{\exp(\bm{\bar{h}}_i\bm{\bar{h}}_1^T/\tau) + \sum_{k\notin C(i)} \exp(\bm{\bar{h}}_i\bm{\bar{h}}_k^T/\tau)}] =0 \label{false_negative_1}
\end{align}
As ($\bm{x}_i$,$\bm{x}_1$) is a true positive pair, both ($\bm{x}_i$, $\bm{x}_k$) and ($\bm{x}_1$, $\bm{x}_k$) are true negative pairs. According to Definition~\ref{definition_1}, $\exp(\bm{\bar{h}}_i\bm{\bar{h}}_k^T/\tau) \approx 0$ and $\exp(\bm{\bar{h}}_1\bm{\bar{h}}_k^T/\tau) \approx 0$ for some positive small values $\tau$ and both two terms of Eq.~\ref{false_negative_1} is 0. Thus, Eq.~\ref{false_negative_1} holds. 

Next, we want to show that if there exists one false positive sample, we reach a contradiction. Assuming that $\bm{x}_1$ is one false positive sample of $\bm{x}_i$ in the batch, then both ($\bm{x}_1$, $\bm{x}_j$) and ($\bm{x}_i$, $\bm{x}_1$) are false positive pairs and ($\bm{x}_1$, $\bm{x}_k$) is a false negative pair for some $k$ (\eg, $\bm{x}_1$ and $\bm{x}_k$ are from the same cluster for some $k$). Similarly, as ($\bm{x}_i$, $\bm{x}_k$) is a true negative pair for all $k$, $\exp(\bm{\bar{h}}_i\bm{\bar{h}}_k^T/\tau) \approx 0$ and the second term of Eq.~\ref{false_negative_1} is approximately 0. 
Therefore, we have
\begin{align}
    & \nonumber \sum_{j\in C(1), j\neq 1} \frac{\sum_{k \notin C(1)} [\exp(\bm{\bar{h}}_1\bm{\bar{h}}_k^T/\tau)\bm{\bar{h}}_j - \exp(\bm{\bar{h}}_1\bm{\bar{h}}_k^T/\tau) \bm{\bar{h}}_k]}{\exp(\bm{\bar{h}}_1\bm{\bar{h}}_j^T/\tau) + \sum_{k\notin C(1)} \exp(\bm{\bar{h}}_1\bm{\bar{h}}_k^T/\tau)} + 0 =0 \\
    & \sum_{j\in C(1), j\neq 1}\sum_{k \notin C(1)} [\exp(\bm{\bar{h}}_1\bm{\bar{h}}_k^T/\tau)\bm{\bar{h}}_j - \exp(\bm{\bar{h}}_1\bm{\bar{h}}_k^T/\tau) \bm{\bar{h}}_k] = 0
    \label{false_negative_2}
\end{align}
We multiply Eq.~\ref{false_negative_2} by $\bm{\bar{h}}_i^T$, where ($\bm{x}_i$, $\bm{x}_j$) is a true positive pair for any $j$ (\ie, both $\bm{x}_i$ and $\bm{x}_j$ are from the same cluster) and ($\bm{x}_i$, $\bm{x}_k$) is a true negative pair for any $k$. Then, we have
\begin{align}
    &\nonumber \sum_{j\in C(1), j\neq 1}\sum_{k \notin C(1)} [\exp(\bm{\bar{h}}_1\bm{\bar{h}}_k^T/\tau)\bm{\bar{h}}_j \bm{\bar{h}}_i^T \\
    & + \exp(\bm{\bar{h}}_1\bm{\bar{h}}_k^T/\tau) (-\bm{\bar{h}}_k \bm{\bar{h}}_i^T)] =0
    \label{false_negative_3}
\end{align}
Since ($\bm{x}_i$, $\bm{x}_j$) is a true positive pair and ($\bm{x}_i$, $\bm{x}_k$) is a true negative pair, we have $\exp(\bm{\bar{h}}_j\bm{\bar{h}}_i^T/\tau) >1$ and $\exp(\bm{\bar{h}}_k\bm{\bar{h}}_i^T/\tau) \approx 0$, which means that $\bm{\bar{h}}_j\bm{\bar{h}}_i^T>0$ and $\bm{\bar{h}}_k\bm{\bar{h}}_i^T <0$. Therefore, both two terms of Eq.~\ref{false_negative_3} are non-negative and Eq.~\ref{false_negative_3} holds if and only if $\exp(\bm{\bar{h}}_1\bm{\bar{h}}_k^T/\tau)\approx 0$ for any $k\notin C(1)$ (\ie, ($\bm{x}_1$, $\bm{x}_k$) is a true negative pair for any $k\notin C(1)$). 

If ($\bm{x}_1$, $\bm{x}_k$) is a true negative pair for any $k\notin C(1)$, then ($\bm{x}_1$, $\bm{x}_i$) has to be a true positive pair, and it contradicts our assumption that $\bm{x}_1$ is a false positive sample of $\bm{x}_i$. Therefore, we reach a contradiction and we could not get the optimal solution for $\bm{\bar{h}}_1$, which completes the proof. 
\end{proof}

\section{Experiments}
\label{more_results}
In this section, we show the details of generating anomalous node for the semi-supervised datasets, including IMDB and DBLP. We also show the hyper-parameter specification for reproducing the experimental results.

\subsection{Anomalous node generation}
\label{generate_anomaly}
We follow the published works \cite{DBLP:conf/sdm/DingLBL19} to generate anomalous nodes by perturbing the topological structure or node attributes of an attributed network. To perturb the topological structure of an attributed network, we adopt the method introduced by~\cite{DBLP:conf/sdm/DingLBL19} to generate some small cliques as in many real-world scenarios. A small clique is a typical anomalous substructure due to larger node degrees than normal nodes. After we specify the clique size as $m$, we randomly select $m$ nodes from the network and then make those nodes fully connected. Then all the $m$ nodes in the clique are regarded as anomalies. In addition to the injection of structural anomalies, we adopt another attribute perturbation schema introduced by~\cite{DBLP:conf/sdm/DingLBL19} to generate anomalies from an attribute perspective. For each selected node $u_i$, we randomly pick another $k$ nodes and select node $u_j$ whose attributes deviate the most from node $u_j$ among the $k$ nodes by maximizing the Euclidean distance $||x_i-x_j||^2$. Afterward, we then change the attributes $x_i$ of node $u_i$  to $x_j$. 

\subsection{Reproducibility} 
\label{reproducibility}
All of the real-world data sets are publicly available.  The experiments are performed on a Windows machine with a 24GB RTX 4090 GPU. We use TAM as the backbone of our method to capture the local node affinity. We set the number of clusters to be 10, $\lambda=0.1$ and $\alpha=0.8$ for the CERT dataset. For the IMDB dataset, we set the number of clusters to be 10, the value of $\lambda=0.01$ and $\alpha=1$. For the DBLP dataset, we set the number of clusters to be 10, the value of $\lambda=0.01$ and $\alpha=1$. For the BlogCatalog dataset, we set the number of clusters to be 5, $\lambda=0.01$ and $\alpha=0.01$. For the Amazon dataset, we set the number of clusters to be 10, $\lambda=1$ and $\alpha=0.8$. For the Yelp dataset, we set the number of clusters to be 10, $\lambda=1$ and $\alpha=0.8$.

\end{document}